%% file: arxiv_deepthink/main_deepthink_template.tex
\documentclass[11pt,letterpaper]{article}

\usepackage{amsmath,amssymb,amsthm}
\usepackage{deepthink}
\usepackage{microtype}
\usepackage{graphicx}
\usepackage{subcaption}
\usepackage{booktabs} 
\usepackage{hyperref}

\usepackage{amsmath}
\usepackage{amssymb}
\usepackage{mathtools}
\usepackage{amsthm}

\usepackage[nameinlink]{cleveref}
\usepackage[
  backend=biber,
  style=alphabetic,
  maxbibnames=100,
  minbibnames=100,
  maxcitenames=2,
  mincitenames=2
]{biblatex}
\newtheorem{lemma}{Lemma}
\newtheorem{thm}{Theorem}

\newtheorem{defi}{Definition}

\newtheorem{assumption}{Assumption}

\newtheorem{exam}{Example}

\usepackage[textsize=tiny]{todonotes}
\input{arxiv_deepthink/macro}

\title{Understanding Generalization in Diffusion Distillation via Probability Flow Distance}

\newcommand{\corrauth}{\textsuperscript{\ddag}}

\authorblock{
  \href{https://www.huijiezh.com/}{Huijie Zhang}\textsuperscript{1},
  \href{https://zjhuangg.github.io/}{Zijian Huang}\textsuperscript{1},
  \href{https://chicychen.github.io/}{Siyi Chen}\textsuperscript{1},
  \href{https://scholar.google.com/citations?user=O3Df4PwAAAAJ&hl=en}{Jinfan Zhou}\textsuperscript{3},
  \href{https://la0ka1.github.io/}{Zekai Zhang}\textsuperscript{1},
  \href{https://peng8wang.github.io/}{Peng Wang}\textsuperscript{2},
  \href{https://qingqu.engin.umich.edu/}{Qing Qu}\corrauth\textsuperscript{1}
}

\affiliation{
  \textsuperscript{1}University of Michigan \quad $\cdot$ \quad \textsuperscript{2}University of Macau \quad $\cdot$ \quad \textsuperscript{3}University of Chicago
}

\authornote{
  \corrauth\ Corresponding author
}

\abstracttext{
\noindent Diffusion distillation provides an effective approach for learning lightweight and few-steps diffusion models with efficient generation.
However, evaluating their generalization remains challenging: theoretical metrics are often impractical for high-dimensional data, while no practical metrics rigorously measure generalization.
In this work, we bridge this gap by introducing probability flow distance ($\PFD$), a theoretically grounded and computationally efficient metric to measure generalization. 
Specifically, $\PFD$ quantifies the distance between distributions by comparing their noise-to-data mappings induced by the probability flow ODE.
Using $\PFD$ under the diffusion distillation setting, we empirically uncover several key generalization behaviors, including:
(1) quantitative scaling behavior from memorization to generalization,
(2) epoch-wise double descent training dynamics, and
(3) bias-variance decomposition. Beyond these insights, our work lays a foundation for generalization studies in diffusion distillation and bridges them with diffusion training.
}


\date{\today}
\correspondence{\href{mailto:huijiezh@umich.edu}{huijiezh@umich.edu}}
\resources{\quad \href{https://www.huijiezh.com/probability-flow-distance/index.html}{Project page} \quad $\cdot$ \quad Code (TBD)}

\headerlogo{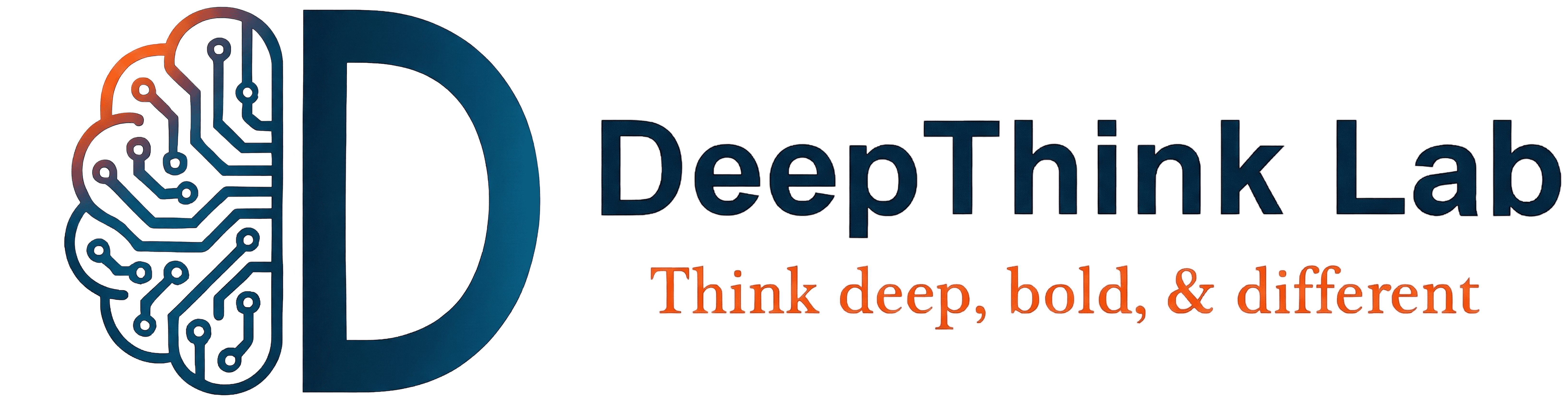}{https://deepthink-umich.github.io}

\begin{document}

\makeDeepthinkHeader

\begin{figure}[ht]
\centering
\includegraphics[width=\linewidth]{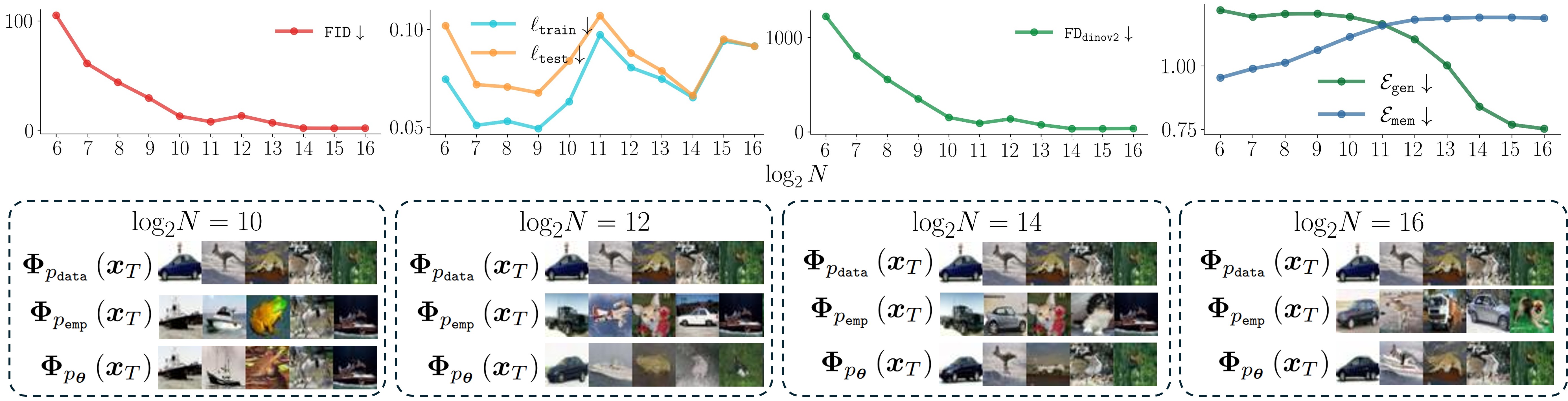}
\caption{\textbf{Comparison of practical metrics on the MtoG transition.} The top figure plots multiple evaluation metrics as functions of $\log_2 N$. The bottom figure visualizes the generation when $N = 2^{10}, 2^{12}, 2^{14}, 2^{16}$, sampled from the $\pdata$ (top row), the $\pemp$ (middle row), and $p_{\bm \theta}$ (bottom row). The same column shared the same initial noise $\bm x_T$.}
\label{fig:metric-comparison-selected}
\end{figure}

\newpage
\tableofcontents

\newpage


\input{arxiv_deepthink/sections/intro}
\input{arxiv_deepthink/sections/definition}

\input{arxiv_deepthink/sections/quantify_generalization_error}
\input{arxiv_deepthink/sections/generalization_behavior}
\input{arxiv_deepthink/sections/discussion}

\printbibliography

\newpage

\appendix

The appendix is organized as follows. We first discuss related work in \Cref{sec:related-work}. Next, we provide detailed proofs for \Cref{sec:pf-distance} in \Cref{app:proof}. Experimental settings and additional discussions for \Cref{sec:gen-eval} and \Cref{sec:gen-behavior} are presented in \Cref{app:exp}. We then offer further discussion related to $\PFDm$ in \Cref{app:estimate_m_score}. Finally, ablation studies for $\PFD$ are included in \Cref{app:ablation_study}.

\input{arxiv_deepthink/appendix/related_works}
\input{arxiv_deepthink/appendix/proof}
\input{arxiv_deepthink/appendix/experiments}
\input{arxiv_deepthink/appendix/PFDm_discussion}
\input{arxiv_deepthink/appendix/ablation_study}

\end{document}

%% file: arxiv_deepthink/macro.tex
\usepackage{color}
\usepackage[mathscr]{euscript}
\usepackage{enumitem}
\usepackage{pifont}   

\usepackage[overload]{empheq}
\usepackage{url}            
\usepackage{bm}
\usepackage{makecell}
\usepackage{pifont}
\newcommand{\cmark}{\ding{51}} 
\newcommand{\xmark}{\ding{55}} 
\newcommand{\beqa}{\begin{eqnarray}}
\newcommand{\eeqa}{\end{eqnarray}}
\newcommand{\beqas}{\begin{eqnarray*}}
\newcommand{\eeqas}{\end{eqnarray*}}
\newcommand{\ba}{\begin{array}}
\newcommand{\tr}{\mathrm{Tr}}
\newcommand{\ea}{\end{array}}
\newcommand{\bi}{\begin{itemize}}
\newcommand{\ei}{\end{itemize}}

\newcommand{\PFD}{\mathtt{PFD}}
\newcommand{\pemp}{p_{\mathtt{emp}}}
\newcommand{\pdata}{p_{\mathtt{data}}}
\newcommand{\PFDm}{\mathcal E_\mathtt{mem}}
\newcommand{\PFDg}{\mathcal E_\mathtt{gen}}
\newcommand{\PFDb}{\mathcal E_\mathtt{bias}}
\newcommand{\PFDv}{\mathcal E_\mathtt{var}}

\DeclareMathOperator{\diag}{diag}

\newcommand{\teacher}{\bm \theta_{\mathtt{t}}}

\newcounter{spb}
\def\b0{\bm{0}}
\def\b1{\bm{1}}
\def\ba{\bm{a}}

\def\E{\mathbb{E}}

\def\P{\mathbb{P}}

%% file: arxiv_deepthink/sections/intro.tex
\section{Introduction}
\label{sec:intro}



While diffusion models \cite{scoreSDE, FM, RF, ddpm} have revolutionized generative AI, their real-world deployment remains limited by substantial computational cost. Knowledge distillation \cite{hinton2015distilling}, which transfers capabilities from a high-capacity teacher to a significantly smaller student, has therefore become a standard approach for practical diffusion-model deployment. Distillation has shown impressive gains in efficiency through both lightweight student architectures \cite{chen2025snapgen, kim2024bk, hu2026snapgen++} and shorter sampling trajectories \cite{yin2024one, yin2024improved, salimans2022progressive, song2023consistency, meng2023distillation, sauer2024adversarial, xie2024em}. 
Beyond efficiency, recent study shows that distillation can also mitigate memorization \cite{borkar2026memorization}, a known issue in diffusion model training that raises copyright and privacy concerns \cite{gu2025on, carlini2023extracting, somepalli2023diffusion, somepalli2023understanding}. In some cases, diffusion distillation even leads to surprising improvements in generation quality \cite{ma2025diffusion}.

However, despite the remarkable performance of distilled models, a comprehensive understanding of their generalization ability and working mechanisms remains elusive. Specifically, existing metrics for evaluating the generalizability of diffusion distillation suffer from significant shortcomings: common empirical metrics like Fréchet inception distance ($\mathtt{FID}$) \cite{heusel2017gans}, Inception Score ($\mathtt{IS}$) \cite{salimans2016improved} focus on generation quality, but they cannot distinguish between memorization and generalization, where both can yield high-quality outputs. Neural Network Divergence ($\mathtt{NND}$) \cite{arora2017generalization, gulrajani2018towards} proposed to measure the generalizability for generative adversarial networks (GANs) \cite{goodfellow2014generative}. However, it requires a large amount of data for evaluation and is not suitable for diffusion distillation. Although recent works measure generalization by evaluating the likelihood of generated samples that are copied from the training data \cite{zhang2023emergence, yoon2023diffusion}, this can be misleading, as pure noise may be misclassified as generalized output. 

On the other hand, other approaches aim to measure generalization by comparing the distance between the student distribution and the teacher distribution. While distributional distance such as Kullback-Leibler divergence ($\mathtt{KL}$) \cite{chen2023restoration, nie2024the, li2023generalization}, total variation ($\mathtt{TV}$) \cite{chen2023the, li2024towards, li2024accelerating, yang2024leveraging}, and 2-Wasserstein distance ($W_2$) \cite{gao2025convergence, debortoli2022convergence, chen2023improved, Gao2023WassersteinCG} are theoretically appealing, they are often computationally expensive and thus impractical for diffusion distillation. (We defer a detailed discussion in \Cref{sec:compare_theoretical_metric}.)
As summarized in \Cref{tab:measure_comparison}, existing metrics are not both accurate and efficient for evaluating diffusion distillation in practice, highlighting the need for a generalization metric that is both theoretically grounded and practically tractable.

Moreover, developing a principled evaluation framework within the distillation setting is imperative. Such a framework is crucial not only for deepening our insight into the mechanisms underlying diffusion distillation but also for providing systematic guidance in designing more effective architectures, training strategies, and practical benchmarks.

\noindent \textbf{Our Contribution.} 
In this work, we propose a novel metric,  termed the probability flow distance ($\PFD$), that can faithfully evaluate the generalization ability of distilled diffusion models. Specifically, our $\PFD$ metric quantifies distributional differences by leveraging the backward probability flow ODE (PF-ODE) \cite{scoreSDE}, which is widely used in the sampling process of diffusion models. 
Unlike practical metrics such as $\mathtt{FID}$, which cannot distinguish between memorization and generalization, our $\PFD$ provides a theoretically grounded measure of distance between distributions, offering a more reliable assessment of generalization. 
Compared to theoretical metrics like the Wasserstein distance, $\PFD$ is computationally more efficient by leveraging the benign properties of PF-ODE. 
Moreover, by leveraging $\PFD$, our analysis reveals several intriguing generalization phenomena that offer new insights into the learning behavior of diffusion distillation, which could potentially be informative for training diffusion models in general.

\begin{itemize}[leftmargin=*]
    \item \textbf{Scaling behavior from memorization to generalization.} Our metric quantitatively characterizes the scaling behavior of diffusion distillation in the transition from memorization to generalization. Specifically, we demonstrate that generalization follows a universal scaling behavior governed by $N / \sqrt{|\bm \theta|}$, where $N$ is the training dataset size and $|\bm \theta|$ is the number of model parameters. In contrast, prior studies \cite{zhang2023emergence, yoon2023diffusion} have only considered the effects of model capacity or dataset size in isolation, without capturing their joint influence on generalization.
     \item \textbf{Understanding generalization in learning dynamics.} Our $\PFD$ metric reveals and reconciles several intriguing generalization phenomena within the learning dynamics. Specifically, under regimes of sufficient training data, we identify an ``\textbf{epochwise double descent}'' behavior in diffusion distillation, where the generalization error initially decreases, subsequently increases, and finally decreases again toward convergence. While similar phenomena have been observed in overparameterized supervised models, we provide the first empirical validation of this behavior in the context of diffusion distillation. Conversely, when training data is limited, we observe early learning, where models initially generalize but later transition to memorization until convergence.

    \item \textbf{The bias-variance trade-off of generalization errors.} Finally, we show that our $\PFD$ metric naturally introduces a bias–variance decomposition of the generalization error, extending classical statistical learning theory to diffusion models. Empirically, we observe a trade-off consistent with supervised learning: increasing model capacity reduces bias but increases variance, yielding a characteristic U-shaped generalization error curve. This finding could potentially help us identify overfitting and memorization in the distillation and training of diffusion models.
\end{itemize}

%% file: arxiv_deepthink/sections/definition.tex
\section{Introduction of Probability Flow Distance}\label{sec:pf-distance}

In this section, we propose a new metric called probability flow distance ($\PFD$), which is designed to quantify the distance between two arbitrary probability distributions. 
The design of $\PFD$ is motivated by the PF-ODE, which we first review in \Cref{subsec:PF-ODE-review}. 
We then formally define $\PFD$ in \Cref{subsec:def-pfd} and present its empirical estimation with theoretical guarantees in \Cref{subsec:measure-pfd}. 

\subsection{The Mapping Induced by PF-ODE}\label{subsec:PF-ODE-review}

In general, PF-ODE \cite{scoreSDE} is a class of ordinary differential equations (ODE) that aims to reverse a forward process, where Gaussian noise is progressively added to samples drawn from an underlying distribution, denoted as $p_{\texttt{data}}$. The forward process and the PF-ODE can be described as follows:
\begin{itemize}[leftmargin=*]
    \item \emph{Forward process.} Given a sample $\bm x_0 \overset{i.i.d.}{\sim} p_{\texttt{data}}(\bm x)$, the forward process progressively corrupts it by adding Gaussian noise. This process can be characterized by the stochastic differential equation (SDE) $\mathrm{d}\bm x_t =  f(t) \bm x_t \mathrm{d}t + g(t) \mathrm{d} \bm{w}_t,$
, where  $t \in [0, T]$ is the time index, $\{\bm{w}_t\}_{t \in [0, T]} $ is a standard Wiener process, and $f(t), g(t): \mathbb{R}_+ \to \mathbb{R}$ are the drift and diffusion functions that control the noise schedule. In this work, we adopt the noise schedule proposed by elucidated diffusion models (EDM) \cite{EDM}, where $f(t) = 0$ and $g(t) = \sqrt{2 t}$. Substituting this into the SDE and integrating both sides, we obtain $\bm x_t =  \bm x_0 + \int_{0}^{t} \sqrt{2 \tau} \mathrm{d} \bm{w}_\tau.$
For ease of exposition, we use $p_{t}(\bm{x}_t)$ to denote the distribution of the noisy image $\bm x_t$ for each $t\in[0,T]$. In particular, it is worth noting that $p_0(\bm x) = p_{\texttt{data}}(\bm x)$ and $p_T(\bm x) \to \mathcal{N}(\bm 0, T^2\bm I)$ as $T \to +\infty$.  

    \item \emph{Probability flow ODE.} According to \cite{scoreSDE}, the PF-ODE can transform a noise sample $\bm x_{T}$ back into a clean data sample $\bm x_0$. Specifically, under the EDM noise scheduler, the PF-ODE admits the following form: 
 \begin{align}\label{eq:reve_ode_em}
     \mathrm{d}\bm{x}_t =  - t\nabla \log p_{t}(\bm{x}_t) \mathrm{d}t,
 \end{align} 
 where $\nabla_{\bm{x}} \log p_t(\bm{x}_t)$ (or simply $\nabla \log p_t(\bm{x}_t)$) denotes the \emph{score function} of the distribution $p_t(\bm{x}_t)$ at time $t\in[0,T]$. According to \cite{scoreSDE}, the backward PF-ODE \eqref{eq:reve_ode_em} and the forward SDE have the same distribution at each timestep $t$. In practice, since the score function $\log p_{t}(\bm{x}_t)$ is unknown, in diffusion models we approximate it using a neural network $\bm s_{\bm \theta}(\bm x_t, t)$ and employ a numerical solver to generate samples from \Cref{eq:reve_ode_em}. Additional details are provided in \Cref{app:train_diffusion_model}.
\end{itemize}
 

The backward PF-ODE induces a \emph{unique} mapping $\bm \Phi_{p_{\texttt{data}}}$ from $\bm x_T$ to $\bm x_0$. Integrating both sides of \eqref{eq:reve_ode_em} from $T$ to $0$, the mapping $\bm \Phi_{p_{\texttt{data}}}$ can be expressed as:
 \begin{align}\label{eq:n2i_integral}
    \bm \Phi_{p_{\texttt{data}}}(\bm x_T) \coloneqq \bm x_T - \int_{T}^{0}t\nabla \log p_{t}(\bm{x}_t) \mathrm{d}t. 
\end{align}
Previous work \cite{scoreSDE} shows that, when $\bm{x}_T \sim \mathcal{N}(\bm{0}, T^2 \bm{I})$ and $T \to +\infty$, the random variable $\bm{\Phi}_{p_{\texttt{data}}}(\bm{x}_T)$ is distributed according to $p_{\texttt{data}}(\bm{x})$. Consequently, if the data distribution $p_{\texttt{data}}$ is known, the score function $\nabla \log p_t(\bm{x}_t)$ is explicitly available, and the backward PF-ODE induces a deterministic and unique mapping from the Gaussian distribution to $p_{\texttt{data}}$ (we formally prove uniqueness in the identity property of our \Cref{thm:properties}). 

\subsection{Definition of Probability Flow Distance}\label{subsec:def-pfd}

Based on the above setup, we define a metric to measure the distance between any two distributions as follows.
\begin{defi}[Probability flow distance ($\PFD$)]\label{def:pfd}
For any two given distributions $p$ and $q$ of the same dimension, we define their distribution distance as
\begin{equation}\label{eq:pfd}
\scalebox{1}{$
    \PFD\left(p, q\right) \coloneqq
    \left( \mathbb  E_{\bm x_T} \left[\left\| \bm \Psi \circ \bm \Phi_{p} \left(\bm x_T\right) - \bm \Psi \circ \bm \Phi_{q}  \left(\bm x_T\right) \right\|^2_2\right] \right)^{1/2}.
$}
\end{equation}
Here, $\bm x_T \sim \mathcal{N}(0, T^2  \bm I)$,  $\bm{\Phi}_{p}$ and $\bm{\Phi}_{q}$ denote the mappings between the noise and image spaces for distributions $p$ and $q$, respectively, as defined in \eqref{eq:n2i_integral}, and $\bm{\Psi}(\cdot)$ represents an image descriptor.
\end{defi}
Intuitively, when the image descriptor $\bm{\Psi}(\cdot)$ is an \emph{identity mapping}, $\PFD$ measures the distance between two distributions $p$ and $q$ by comparing their respective noise-to-image mappings $\bm{\Phi}_{p}(\cdot)$ and $\bm{\Phi}_{q}(\cdot)$ starting from the same Gaussian noise input $\bm{x}_T$. 
Small values of $\PFD$ indicate that the two distributions generate similar data from identical noise and, therefore, are close to each other. 

However, in practice, measuring distributional distance alone does \emph{not} fully capture the perceptual quality of natural images. To incorporate perceptual quality, we employ an image descriptor $\bm \Psi(\cdot)$, which is typically implemented using a pre-trained neural network such as DINOv2 \cite{dinov2} or Self-Supervised Copy Detection Descriptor (SSCD) \cite{sscd}. Measuring distances in such feature spaces is a common practice in existing metrics for generative models~\cite{heusel2017gans, salimans2016improved, stein2024exposing}, as it tends to align better with human perception~\cite{stein2024exposing}. As shown by our ablation studies in \Cref{app:diff_embedding_feature}, the $\PFD$ metric with an image descriptor effectively captures differences in perceptual quality while models are generalizing, where a smaller quantity indicates better perceptual quality. Therefore, for all experiments in \Cref{sec:gen-eval,sec:gen-behavior}, we use an image descriptor (e.g., SSCD) for $\bm{\Psi}(\cdot)$.

Nonetheless, for simplicity and analytical tractability, we assume the image descriptor $\bm \Psi(\cdot)$ to be an identity mapping in the following theoretical analysis.
Specifically, under \Cref{def:pfd}, we show that $\PFD$ satisfies the axioms of a metric (Definition 2.15 in \cite{rudin2021principles}).

\begin{thm}\label{thm:properties} 
For any two distributions $p$ and $q$, the $\textup{\texttt{PFD}}$ satisfies the following properties: 
\begin{itemize}[leftmargin=*]
\item (Positivity)  $\PFD(p, q) > 0$ for any $p \neq q$.  
\item (Identity Property) $\PFD(p, q) = 0$ if and only if $p = q$.  
\item (Symmetry) $\PFD(p, q) = \PFD(q, p)$.
\item (Triangle Inequality) $\PFD(p, q) \leq \PFD(p, p') + \PFD(p', q)$ for all $p'$. 
\end{itemize} 
\end{thm}
We defer the proof to \Cref{app:proof}. Note that \Cref{thm:properties} establishes the theoretical validity of $\PFD$ as a metric for measuring the distance between any two probability distributions. 


\subsection{Empirical Estimation of \texorpdfstring{$\PFD$}{PFD}}\label{subsec:measure-pfd}
In practice,  the expectation in \eqref{eq:pfd} is intractable due to the complexity of the underlying distributions. Thus, we approximate the $\PFD$ using finite samples: 
\begin{equation}\label{eqn:empirical-pfd}
\scalebox{1}{$\hat{\PFD}(p, q) = \left(\dfrac{1}{M} \sum\limits_{i = 1}^{M} \left\| \bm \Phi_{p} \left(\bm x^{(i)}_T\right) - \bm \Phi_{q}  \left(\bm x^{(i)}_T\right) \right\|^2_2\right)^{1/2}.
$}
\end{equation} 
Here, $\hat{\PFD}(p, q)$ is the empirical version of $\PFD(p, q)$ computed over $M$ independent samples $\{\bm{x}^{(i)}_T\}_{i=1}^{M} \overset{i.i.d.}{\sim} \mathcal{N}(0, T^2 \bm{I})$ with $T \rightarrow \infty$.

Specifically, our finite-sample approximation relies on two key assumptions: (i) the score functions are smooth at all timesteps, and (ii) the score functions of two distributions remain uniformly close within a bounded region of the input space, which can be described as follows.

\begin{assumption}\label{assump:lip_diff_score}
Let $p$ and $q$  be two distributions with the same dimension, where we assume:

(i) (Lipschitz score functions)
There exists a constant $L > 0$ such that for all $\bm x_1, \bm x_2$ and $t \in [0, T]$, it holds that
\begin{equation}
\scalebox{1}{$
    \left\|\nabla_{\bm x} \log p_{t}(\bm x_1) - \nabla_{\bm x} \log p_{t}(\bm x_2)\right\|_2 \leq L \left\|\bm x_1 - \bm x_2\right\|_2,
$}
\end{equation}
and similarly for \( q_t \).

(ii) (Uniform Closeness) 
For all $t \in [0, T]$, there exists a constant $\epsilon > 0$ such that 
\begin{equation}
    \left\|\nabla_{\bm x} \log p_{t}(\bm x) - \nabla_{\bm x} \log q_{t}(\bm x)\right\|_2 \leq \epsilon. 
\end{equation}
\end{assumption}
The Lipschitz continuity of the score function is a common assumption widely adopted in the theoretical analysis of score functions in diffusion models \cite{block2020generative,lee2022convergence,chen2023sampling,chen2023improved,zhu2023sample,chen2023score}. More recently, this property has been rigorously established under the assumption that the data distribution is a mixture of Gaussians \cite{liang2024unraveling}. The uniform closeness assumption holds when $p$ and $q$ both follow \Cref{assump:lip_diff_score} (\emph{i}) and have support on compact domains, which is often the case for image distributions. Under \Cref{assump:lip_diff_score}, the concentration of the empirical estimate $\hat{\PFD}(p, q)$ to $\PFD(p, q)$ can be characterized as follows.

\begin{thm}\label{thm:empirical_approximation}
Let $p$ and $q$ be two distributions satisfying \Cref{assump:lip_diff_score}, and let $\hat{\PFD}(p, q)$ denote the empirical estimate of $\PFD(p, q)$ using $M$ independent samples, as defined in \eqref{eqn:empirical-pfd}. Then, for any $\gamma > 0$, this empirical estimate satisfies the following bound:
\begin{equation}
\scalebox{1}{$
\left| \hat{\PFD}(p, q) - \PFD(p, q)  \right| \leq \gamma
\text{ whenever}\; M \geq \dfrac{\kappa^4(L, \epsilon)}{2 \gamma^4} \log \dfrac{2}{\eta},
$}
\end{equation}
with probability at least $1 - \eta$. Here, $\kappa(L, \epsilon) \coloneqq \exp\left(\frac{L T_\xi^2}{2}\right) \xi + \frac{\epsilon}{L}\left(\exp\left(\frac{L T_\xi^2}{2}\right) - 1\right)$ is a constant, with a numerical constant $\xi>0$  and a finite timestep $T_\xi$ depending only on $\xi$.
\end{thm}
We defer the proof to \Cref{app:proof}. Given the score functions of both distributions are smooth and uniformly close, our result in \Cref{thm:empirical_approximation} guarantees that $\PFD(p, q)$ can be approximated to arbitrary precision by its empirical estimate $\hat{\PFD}(p, q)$ with high probability, given a finite number of samples.

\begin{table}[h]
\centering
\caption{\textbf{Metrics comparison.}  }
\begin{tabular}{lcc}
\toprule
 & Efficient & Accurate \\
\midrule
\textit{Theoretical distances} \\
Density-based ($\mathtt{KL}$, $\mathtt{TV}$)  & \xmark & \cmark \\
Sample-based ($W_2$, $\mathtt{MMD}$) & \xmark  & \cmark \\
\midrule
\textit{Practical metrics} \\
$\mathtt{FID}$, $\mathtt{IS}$, $\mathtt{NND}$, \emph{etc.} & \cmark & \xmark \\
\midrule
$\PFD$ (\textbf{Ours}) & \cmark & \cmark \\
\bottomrule
\end{tabular}%
\label{tab:measure_comparison}
\end{table}

\subsection{Advantages of \texorpdfstring{$\PFD$}{PFD} over Existing Metrics.}

\label{sec:compare_theoretical_metric}
As summarized in \Cref{tab:measure_comparison}, we conclude this section by highlighting the advantages of the proposed $\PFD$ over commonly used theoretical metrics for measuring distributional distances, including both density-based and sample-based methods. A comparison with practical evaluation metrics is presented at the end of \Cref{sec:gen-eval}.

\noindent \textbf{Sampling efficiency.} We compare the sampling efficiency of $\PFD$, $\mathtt{FID}$, and $W_2$ on Gaussian distributions, as shown in \Cref{fig:compare_sample_efficiency}, using the experimental settings described in \Cref{app:compare_synthetic_metric}. With the same number of estimated samples ($M = 4096$), $\PFD$ attains a relative error of approximately $4 \times 10^{-3}$, whereas $\mathtt{FID}$ and $W_2$ achieve only about $2 \times 10^{-2}$. Consequently, to reach the same level of relative error, $\PFD$ requires substantially fewer samples. 

\noindent \textbf{Computational efficiency.} Even with the same number of estimated samples $M$, $\PFD$ requires less computation time than other methods:
\begin{itemize}[leftmargin=*]
    \item \textbf{Comparison with sample-based distances.} Sample-based distances such as the Wasserstein distance and Maximum Mean Discrepancy ($\mathtt{MMD}$) incur $O(M^2)$ computational complexity. In contrast, $\PFD$ requires only $O(M)$ computation, making it substantially more efficient.
    \item \textbf{Comparison with density-based distances.} Density-based distances, such as KL divergence, total variation distance, and Jensen–Shannon divergence, require approximating probability densities through computationally intensive methods like the Skilling-Hutchinson trace estimator \cite{scoreSDE, skilling1989eigenvalues, hutchinson1989stochastic}, which become prohibitively expensive in high-dimensional settings for evaluating diffusion models. In contrast, $\PFD$ directly estimates the distributional distance using the score function inherently learned by diffusion models. Furthermore, density-based metrics are theoretically ill-suited for image data, as probability densities remain undefined outside the image manifold \cite{loaiza-ganem2024deep}.
\end{itemize}

%% file: arxiv_deepthink/sections/quantify_generalization_error.tex
\section{Quantifying Generalization Errors via \texorpdfstring{$\PFD$}{PFD}} \label{sec:gen-eval}

In this section, we leverage the $\PFD$ metric in \Cref{sec:pf-distance} to rigorously define and evaluate the generalization error of diffusion models under distillation settings. Specifically, this metric enables us to distinguish between memorization and generalization behaviors, as well as analyze the transition from memorization to generation (MtoG).

This MtoG transition has been explored in recent studies for training diffusion models \cite{yoon2023diffusion,zhang2023emergence,kadkhodaie2023generalization, bonnaire2025why}, which highlight two learning regimes of diffusion models depending on dataset size and model capacity: (\emph{i}) \textbf{Memorization regime:} Large models trained on small datasets memorize the empirical distribution $\pemp(\bm x)$ of the training data, yielding poor generalization and no novel samples. (\emph{ii}) \textbf{Generalization regime:} For fixed model capacity, as the number of training samples increases, the model transitions into generalization, approximating the true data distribution $p_{\mathtt{data}}(\bm{x})$ and generating new samples.

However, existing approaches \cite{yoon2023diffusion,zhang2023emergence,alaa2022faithful} quantify generalization solely by measuring the distance from a generated sample to its nearest neighbor in the training dataset. While effective for identifying memorization, these sample-level measures are inadequate for capturing true generalization because they do not quantify distributional distance. Simply showing that an output is distinct from training data is insufficient; for instance, such metrics could erroneously classify a sample of pure noise as a valid generated sample. To address these limitations, we leverage the $\PFD$ metric to assess generalization at the distributional level. Specifically, we quantify how closely the distribution learned via diffusion models, $p_{\bm \theta}$, approximates the underlying distribution $p_{\mathtt{data}}(\bm x)$ and how closely it aligns with the empirical distribution $\pemp(\bm x)$. Based on this approach, we formally define generalization and memorization errors as follows.


\begin{defi}[Generalization and Memorization Errors]\label{def:PFD_memorization_generalization}
Consider a diffusion model $\bm s_{\bm \theta}$ trained on a finite dataset $\mathcal{D} = \{\bm y^{(i)}\}_{i=1}^N$, where each sample $\bm y^{(i)}$ is drawn i.i.d. from the underlying distribution $p_{\mathtt{data}}(\bm x)$. Denote the learned distribution induced by a diffusion model $\bm s_{\bm \theta}$ as $p_{\bm \theta}(\bm x)$. Using the $\PFD$ metric, we can formally define the generalization and memorization errors as follows:
\begin{align}\label{eqn:gen-error}
\scalebox{.9}{$
   \PFDg\left(\bm \theta\right) \coloneqq \PFD\left( p_{\bm \theta}, p_{\mathtt{data}}\right),\quad \PFDm\left(\bm \theta\right) \coloneqq \PFD\left(p_{\bm \theta},p_{\mathtt{emp}}\right),
$}
\end{align}
where the empirical distribution $\pemp(\bm x) = \frac{1}{N}\sum_{i=1}^N \delta(\bm x - \bm y^{(i)})$, with $\delta(\cdot)$ denoting the Dirac delta function.
\end{defi}
Here, given access to $\pemp(\bm x)$, the memorization error $\PFDm(\bm \theta)$ can be exactly computed (see \Cref{app:estimate_m_score}). We further show that $\PFDm(\bm \theta)$ coincides with metrics introduced in \cite{yoon2023diffusion, zhang2023emergence}. 


\noindent \textbf{Evaluation protocol for generalization.} Under the diffusion distillation setting (see \Cref{sec:gen-behavior}), we study generalization behaviors of a student model distilled from a teacher model. 
Specifically, a pretrained diffusion model $\bm s_{\teacher}(\bm x_t)$ with parameters $\teacher$ serves as the teacher and induces a distribution $p_{\teacher}$, which we treat as the underlying data distribution, i.e., $p_{\mathtt{data}} = p_{\teacher}$. We employ the most straightforward distillation strategy by directly training a student model $\bm s_{\bm \theta}$ on samples drawn from $p_{\teacher}$. Generalization behaviors of the student model is then evaluated by comparing the student-induced distribution $p_{\bm \theta}$ with $p_{\teacher}$ using the generalization error metrics defined in \Cref{def:PFD_memorization_generalization}.


In our experiments for the rest of the paper, both the teacher and student models adopt the U-Net architecture \cite{unet}. The teacher model $\bm s_{\teacher}$ is trained on the CIFAR-10 dataset \cite{krizhevsky2009learning} with a fixed model architecture (UNet-10 introduced in \Cref{app:architecture}). The student model $\bm s_{\bm \theta}$ is trained on samples generated by the teacher, with the number of distillation samples varying from $N = 2^6$ to $N = 2^{16}$, using the same training hyperparameters but different model sizes. For evaluating the generalization error in \eqref{eqn:gen-error}, we compute the $\PFD$ between the teacher and student models using $M = 10^4$ samples drawn from shared initial noise, as defined in \eqref{eqn:empirical-pfd}. Similarly, for the memorization error, we compute the $\PFD$ between the student model and the empirical distribution of the training data. 
Additional details for the evaluation protocol and ablation studies are provided in \Cref{app:evaluate_protocol_setting} and \Cref{app:ablation_study}. 



\noindent \textbf{Comparison with practical metrics for evaluating generalization.}
Before we use the proposed metrics $\PFDg$ and $\PFDm$ for revealing the generalization properties of diffusion models in \Cref{sec:gen-behavior}, we conclude this section by demonstrating their advantages over commonly used practical metrics, such as $\mathtt{FID}$,  $\mathtt{FD}_{\mathtt{DINOv2}}$ \cite{stein2024exposing}, training and testing loss $\ell_{\mathtt{train}}, \ell_{\mathtt{test}}$ (see \Cref{eq:em loss}).
We defer a more comprehensive comparison with other metrics such as $\mathtt{IS}$, $\mathtt{NND}$, $\mathtt{KID}$ \cite{bińkowski2018demystifying}, $\mathtt{CMMD}$ \cite{jayasumana2024rethinking}, $\mathtt{Precision}$, and $\mathtt{Recall}$ \cite{kynkaanniemi2019improved} to \Cref{app:compare_metric}.

As shown in \Cref{fig:metric-comparison-selected}, we compare several metrics for capturing the MtoG transition under distillation. Among them, only the proposed metric $\PFDg$ consistently tracks this transition as the number of distillation samples increases. This conclusion is supported by the qualitative results in the bottom row of \Cref{fig:metric-comparison-selected}: when $N \geq 2^{10}$, increasing $N$ makes the distilled model’s generations (bottom row) visually closer to the teacher distribution (top row), indicating improved generalization. Among all metrics, only $\PFD$ captures this decreasing trend.

From a theoretical perspective, $\mathtt{FID}$, $\mathtt{FD}_{\mathtt{DINOv2}}$ relies on a Gaussian assumption for feature distributions, while $\ell_{\mathtt{train}}$ and $\ell_{\mathtt{test}}$ provide only upper bounds on the negative log-likelihood of the learned distribution $p_{\bm{\theta}}$ \cite{song2021maximum}; as a result, none of these metrics can accurately capture generalization. In contrast, $\PFD$ is both theoretically well-founded and empirically validated as a reliable metric for measuring generalization.


%% file: arxiv_deepthink/sections/generalization_behavior.tex
\section{Findings of Key Generalization Behaviors}\label{sec:gen-behavior}

Based on the evaluation protocol in \Cref{sec:gen-eval}, this section reveals several key generalization behaviors in diffusion distillation: (i) MtoG scaling behaviors with model capacity and distillation dataset size (\Cref{sec:scaling_behavior}), (ii) double descent in learning dynamics (\Cref{sec:training_dynamics}), and (iii) the bias-variance trade-off of generalization error (\Cref{subsec:bias-var}).

\begin{figure}[t]
\begin{center}
    \includegraphics[width = \linewidth]{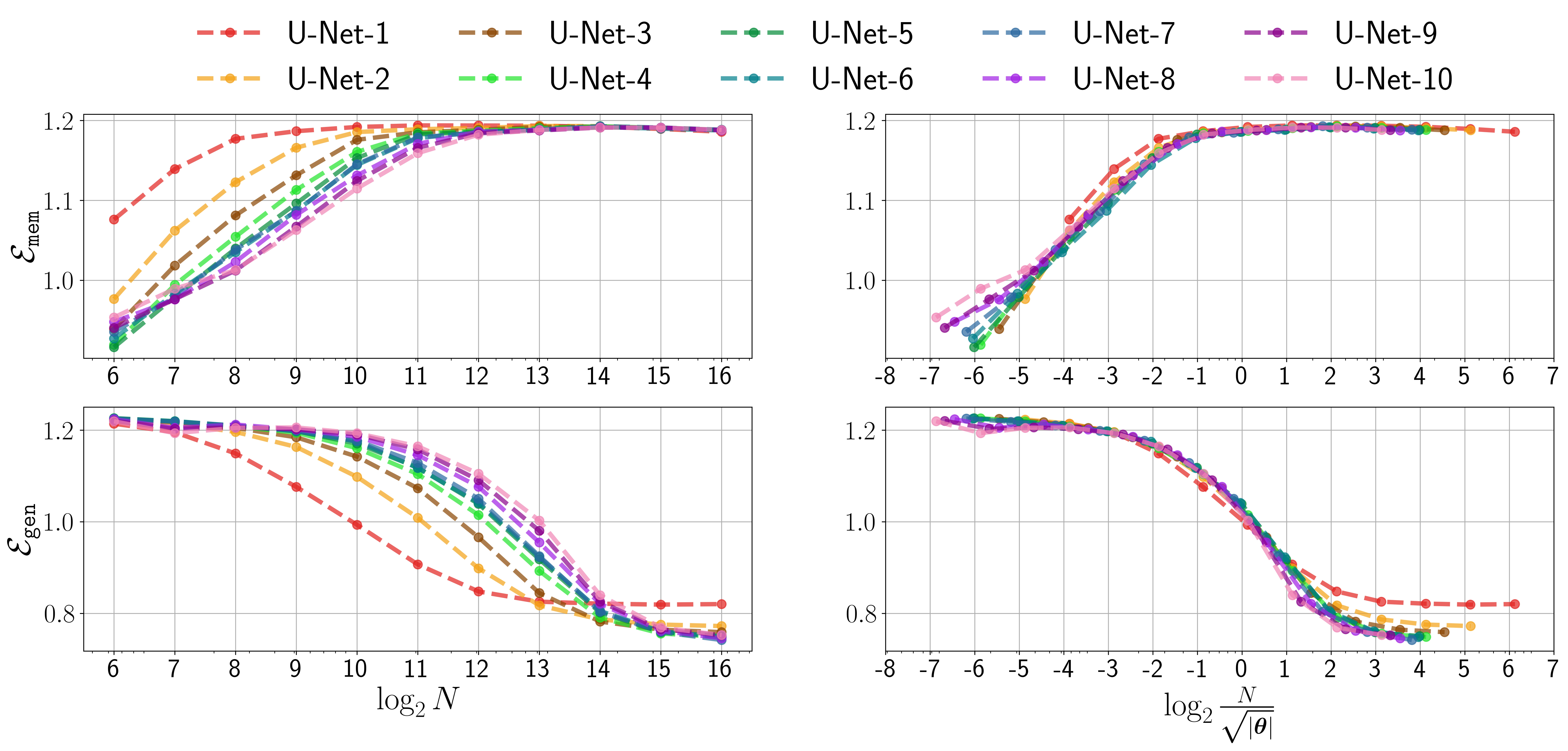}
    \caption{\textbf{Scaling behavior in the MtoG transition.} 
    Memorization error $\PFDm$ and generalization error $\PFDg$ as functions of the number of training samples $N$. Left: errors plotted against $\log_2(N)$ for a family of U-Net architectures (U-Net-1 through U-Net-10). Right: the same quantities plotted against $\log_2(N/\sqrt{|\bm \theta|})$, where $|\bm\theta|$ denotes the number of model parameters.}
    \label{fig:mem2gen}
\end{center}
\end{figure} 

\begin{figure}[h]
\centering
\includegraphics[width=.5\linewidth]{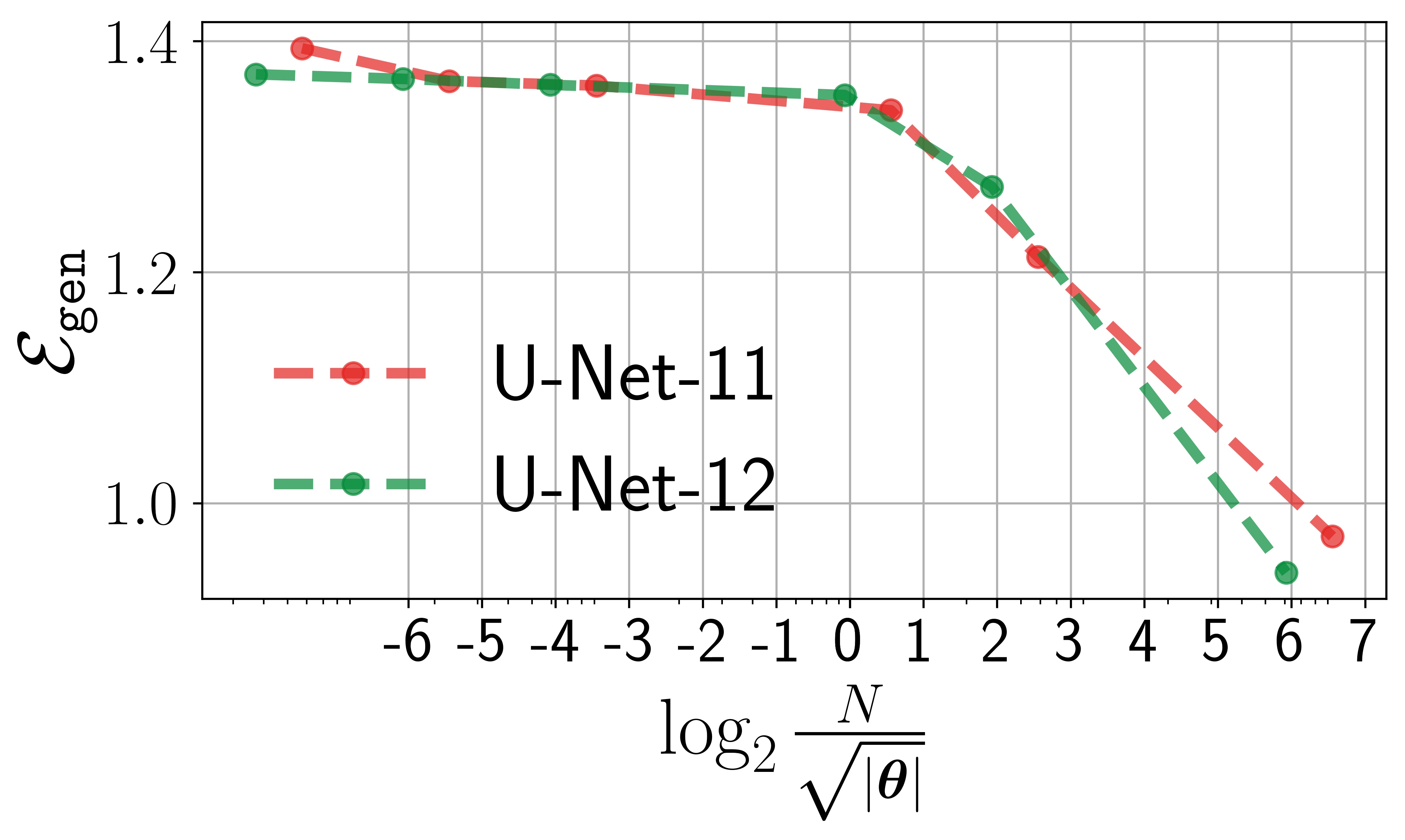}
\caption{\textbf{Scaling behavior over ImageNet dataset.} U-Net-11, U-Net-12 contains $124.2$M and $295.9$M parameters. Distilling data size $N$ ranging from $2^6$ to $2^{20}$.}
\label{fig:mem2gen_imagenet}
\end{figure}

\subsection{Scaling Behaviors of Generalization} \label{sec:scaling_behavior}

We investigate the scaling behavior of the distilled models with respect to both model capacity $|\bm{\theta}|$ and distillation data size $N$, using the metrics $\PFDg$ and $\PFDm$. We evaluate ten U-Net architectures over CIFAR-10 dataset, with model sizes ranging from 0.9M to 55.7M parameters (U-Net-1 to U-Net-10). For each model, we compute $\PFDm$ and $\PFDg$ across varying distillation dataset sizes, following the distillation setting outlined in \Cref{sec:gen-eval}. In addition, we conduct experiments on the ImageNet dataset \cite{deng2009imagenet}. Results for CIFAR-10 and ImageNet are reported in \Cref{fig:mem2gen} and \Cref{fig:mem2gen_imagenet}, respectively. Additional experimental details are provided in \Cref{app:scaling_behavior}. From these results, we observe the following:


\noindent \textbf{MtoG transitions governed by the ratio $N /\sqrt{|\bm \theta|}$.} As shown in \Cref{fig:mem2gen} (left), for a fixed model capacity $|\bm \theta|$, our metrics reveal a clear transition from memorization to generalization as the number of distillation samples $N$ increases, consistent with prior experiments training diffusion models from scratch. \cite{zhang2023emergence, yoon2023diffusion}. Moreover, in contrast to prior work that focuses solely on the effect of distillation sample size $N$, our results in \Cref{fig:mem2gen} (right) for CIFAR-10 and \Cref{fig:mem2gen_imagenet} for ImageNet reveal a consistent quantitative scaling behavior when using our metric, governed by the ratio $N / \sqrt{|\bm \theta|}$ between data size and model capacity. Remarkably, both $\PFDg$ and $\PFDm$ metrics exhibit near-identical MtoG transition curves across models of varying sizes when plotted against this ratio. As such, analogous to the empirical scaling laws observed in large language models~\cite{kaplan2020scaling}, this predictable trend provides practical guidance for the development of diffusion models, particularly when scaling up model size, data, or compute to achieve optimal performance gains.

\subsection{Generalization across Learning Dynamics} \label{sec:training_dynamics}

\begin{figure*}[t]
\begin{center}
    \includegraphics[width = \linewidth]{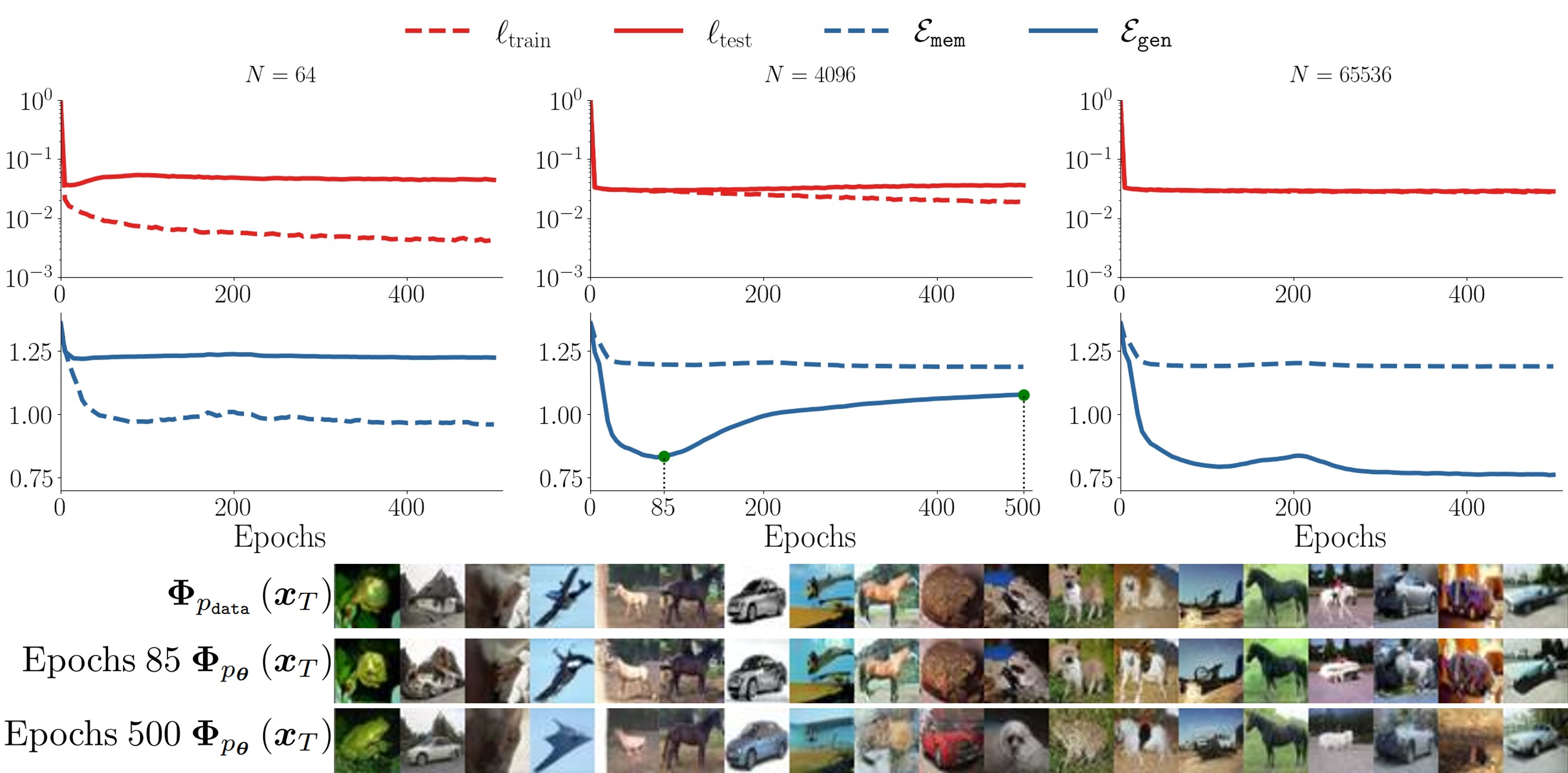}
    \caption{\textbf{Training dynamics of diffusion models in different regimes.} The top figure plots $\PFDm, \PFDg, \ell_{\mathtt{train}}, \ell_{\mathtt{test}}$ over training epochs for different different dataset sizes: $N = 2^6$ (left), $2^{12}$ (middle), $2^{16}$ (right). The bottom figure visualizes the generation when $N = 2^{12}$. The top row shows samples from the underlying distribution $\bm \Phi_{\pdata}(\bm{x}_T)$, while the middle and bottom rows display outputs from the trained diffusion model $\bm \Phi_{p_{\bm \theta}}(\bm{x}_T)$ at epoch 85 and 500, respectively. }\label{fig:training_dynamic_sscd}
\end{center}
\end{figure*} 

Building on the findings in \Cref{sec:scaling_behavior}, we further examine the generalization behavior across different training regimes. Under the distillation setting for CIFAR-10 dataset in \Cref{sec:gen-eval}, we analyze the learning dynamics of a U-Net model with fixed model capacity (UNet-10 introduced in \Cref{app:architecture}) distilled with $N = 2^6$, $2^{12}$, and $2^{16}$, corresponding to the memorization, transition, and generalization regimes in \Cref{sec:scaling_behavior}, respectively. The model is distilled using stochastic gradient descent (SGD) for 500 epochs, during which we track $\PFDm$, $\PFDg$, $\ell_{\mathtt{train}}$, and $\ell_{\mathtt{test}}$ at each epoch. The results in \Cref{fig:training_dynamic_sscd} reveal several notable generalization behaviors that align with phenomena previously observed in the training of overparameterized deep models \cite{zhang2017understanding, nakkiran2021deep}:

\noindent \textbf{Epochwise double descent of the generalization error in the generalization regime.} In contrast, as shown in \Cref{fig:training_dynamic_sscd} (right), distilling in the generalization regime ($N = 2^{16}$) reveals a clear instance of the \emph{double descent} phenomenon~\cite{nakkiran2021deep} in the generalization error. Specifically, the error initially decreases, then increases during intermediate distilling epochs, and finally decreases again as distilling approaches convergence. Notably, this non-monotonic behavior is not captured by the standard training and test losses $\ell_{\mathtt{train}}$ and $\ell_{\mathtt{test}}$, both of which decrease monotonically throughout distilling. This implies that extended distilling can improve generalization performance in the generalization regime.

\noindent \textbf{Epochwise early learning behavior in memorization and transition regimes.} As shown in \Cref{fig:training_dynamic_sscd} (left \& middle), in both the memorization ($N = 2^6$) and transition ($N = 2^{12}$) regimes, the generalization error initially decreases during distilling but reaches its minimum at an early epoch, after which it begins to increase again. This \emph{early learning} (or early generalization) phenomenon becomes more salient as the distillation sample size increases from the memorization to the transition regime. As shown in the visualization at the bottom of \Cref{sec:gen-eval}, the model at Epoch 85 clearly exhibits generalization, whereas the model at Epoch 500 fails to generalize. This is also corroborated by the divergence of training loss $\ell_{\mathtt{train}}$ and test loss $\ell_{\mathtt{test}}$ at the top of the figure. It is worth mentioning that, although early learning behavior has been theoretically and visually demonstrated in previous works \cite{li2024understanding, li2023generalization, bonnaire2025why}, our work develops the first framework for precisely capturing this phenomenon.

\subsection{Bias-variance Trade-off of the Generalization Error}\label{subsec:bias-var}

\begin{figure*}[t]
\begin{center}
    \begin{subfigure}[t]{0.35\linewidth}
        \includegraphics[width=\linewidth]{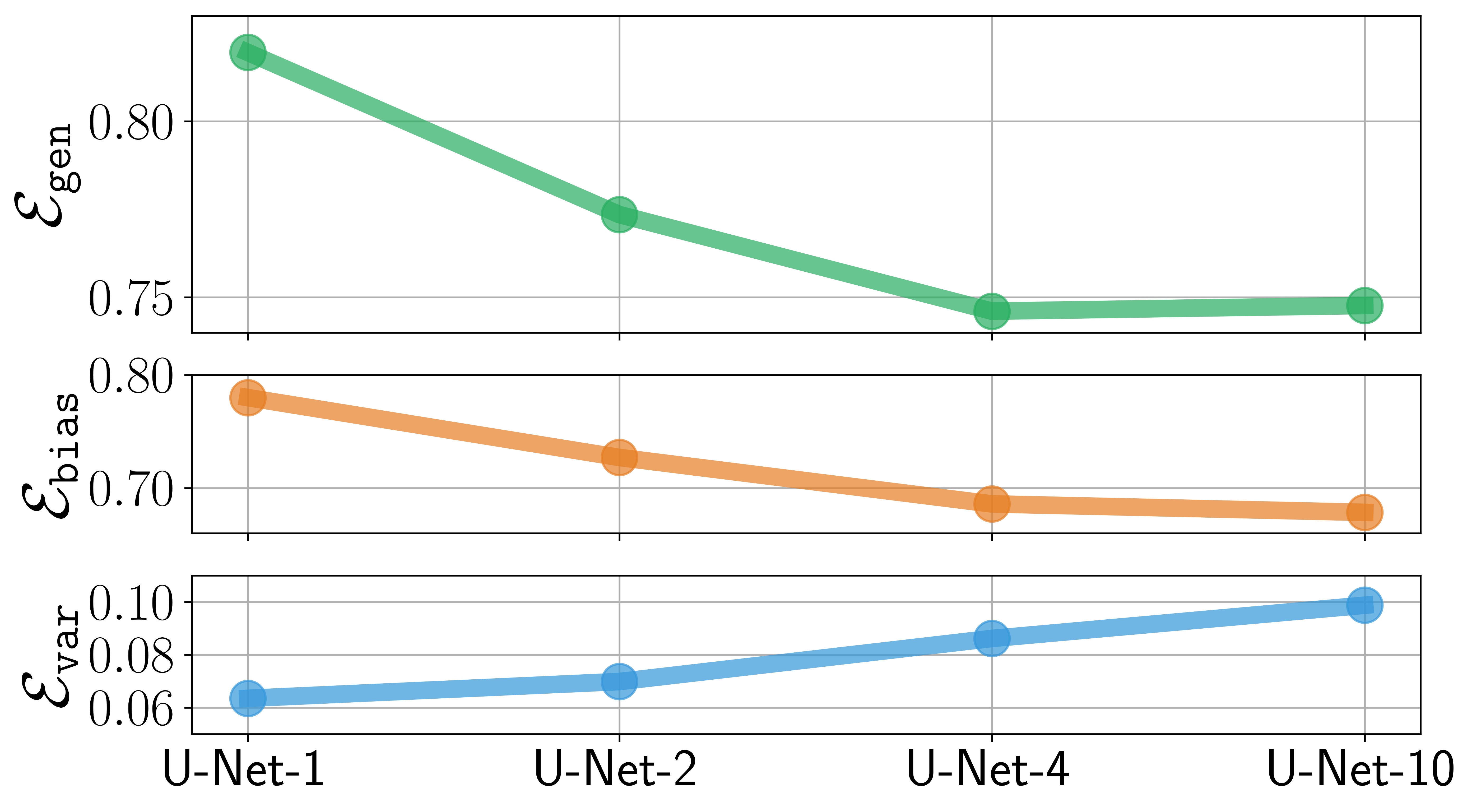}
        \caption{}
    \end{subfigure}
    \begin{subfigure}[t]{0.57\linewidth}
        \includegraphics[width=\linewidth]{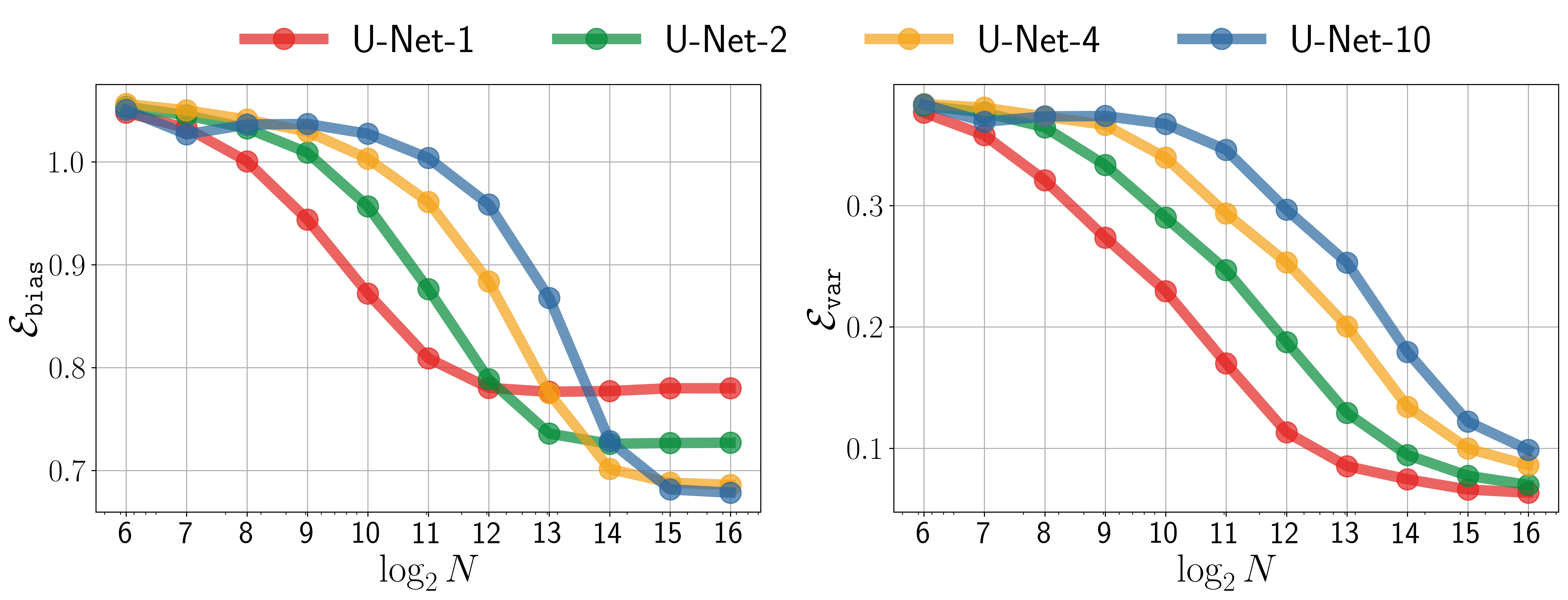}
        \caption{}
    \end{subfigure}
    \caption{\textbf{Bias–Variance Trade-off.} (a) plots the generalization error $\PFDg$, bias $\PFDb$, and variance $\PFDv$ across different network architectures with a fixed distillation sample size of $N = 2^{16}$. (b) shows $\PFDb$ and $\PFDv$ as functions of the number of distillation samples $N$ for various network architectures.}
    \label{fig:bias-variance-decomposition}
\end{center}
\end{figure*} 

In statistical learning theory, bias-variance trade-off is a classical yet fundamental concept in  supervised learning which helps us understand and analyze the sources of prediction error in the model \cite{kohavi1996bias, hastie2009elements,yang2020rethinking, belkin2019reconciling}. Specifically, bias–variance decomposition expresses the expected generalization error as the sum of two components: (i) the \emph{bias term}, which quantifies the discrepancy between the expected model prediction and the true function: high bias indicates systematic error or underfitting; and (ii) the \emph{variance term}, which measures the prediction variability of the model across different training sets: high variance reflects sensitivity to data fluctuations or overfitting.

However, in unsupervised learning settings such as diffusion models, the notion of generalization error was not well-defined prior to our work, in contrast to the well-established definitions in supervised learning. As a result, bias–variance decomposition in this context remains largely unexplored. In this work, we address this gap through the generalization error measure $\PFDg$ (see \Cref{eqn:gen-error}), which admits a bias–variance decomposition analogous to that in the supervised setting, as we detail below.
\begin{defi}[Bias-Variance Decomposition of $\PFDg$]
\label{def:PFD_bias_variance_decomposition}
Based on the same setup as \Cref{def:PFD_memorization_generalization}, 
we can decompose $\PFDg$ in \Cref{eqn:gen-error} as 
\begin{align}\label{def:pfd_bias_variance_decomposition}
    \mathbb E_{\mathcal{D}} \left[\PFDg^2\left(p_{\bm \theta\left(\mathcal{D}\right)}\right)\right] = \PFDb^2 + \PFDv
\end{align}
where $p_{\bm \theta\left(\mathcal{D}\right)}$ denotes the distribution induced by a diffusion model $\bm \theta\left(\mathcal{D}\right)$ trained on a given training dataset $\mathcal{D}$ sampled from $\pdata$. Specifically, the bias and variance terms are defined as:
\begin{align*}
   &\PFDb \coloneqq \mathbb E_{\bm x_T} ([\| \bm \Psi \circ \bm \Phi_{\pdata} (\bm x_T) - \overline{\bm \Psi \circ \bm \Phi}_{p_{\bm \theta}}  (\bm x_T) \|_2^2])^{1/2}, \\
   &\PFDv \coloneqq \mathbb E_{\mathcal{D}} \mathbb E_{\bm x_T} [ \| \bm \Psi \circ \bm \Phi_{p_{\bm \theta(\mathcal{D})}} (\bm x_T) - 
\overline{\bm \Psi \circ \bm \Phi}_{p_{\bm \theta}}  (\bm x_T) \|_2^2],
\end{align*}
with  $\overline{\bm \Psi \circ \bm \Phi}_{p_{\bm \theta}}  (\cdot) \coloneqq \mathbb E_{\mathcal D} [\bm \Psi \circ \bm \Phi_{p_{\bm \theta(\mathcal D)}} (\cdot ) ]$.

\end{defi}

Intuitively, our definitions of the bias term $\PFDb$ and the variance term $\PFDv$ are both well-justified: (i) $\PFDb$ quantifies the systematic error between the learned distribution $p_{\bm \theta}$ and the ground-truth distribution $\pdata$; and (ii) $\PFDv$ captures the variability of model predictions across different training sets by measuring the distance between $p_{\bm \theta}$ and the mean  $\overline{p_{\bm \theta}}$ which can be empirically estimated by averaging over multiple datasets $\mathcal D$ sampled from $\pdata$. 
Experimental results, following the protocol in \Cref{sec:gen-eval}, are shown in \Cref{fig:bias-variance-decomposition}, with detailed settings in \Cref{app:bias-var}.

In \Cref{fig:bias-variance-decomposition} (a), when diffusion models are distilled in the generalization regime, the resulting generalization decomposition aligns with classical bias–variance theory from supervised learning: as model complexity increases, the bias term $\PFDb$ decreases while the variance term $\PFDv$ increases, resulting in a U-shaped generalization error curve. Additionally, \Cref{fig:bias-variance-decomposition} (b) further illustrates the effect of the distillation sample size $N$ and number of parameters $|\bm \theta|$: increasing $N$ reduces both $\PFDb$ and $\PFDv$, thereby lowering the generalization error $\PFDg$, as expected; In contrast, increasing $|\bm \theta|$ consistently increases $\PFDv$, and its effect on $\PFDb$ depends on the size of $N$: it decreases $\PFDb$ when $N \geq 2^{15}$ but increases it when $N \leq 2^{11}$.

%% file: arxiv_deepthink/sections/discussion.tex
\section{Discussion \& Conclusion}
While $\PFD$ is designed to evaluate generalization in distillation settings, we believe its impact extends beyond distillation and is highly relevant to diffusion model training:

\noindent \textbf{Distillation is closely aligned with training from scratch.}  
The teacher-student framework has been widely adopted to understand learning phenomena observed in real-world scenarios, both empirically \cite{betzalel2024evaluation, lee2021continual, saglietti2022analytical} and theoretically \cite{advani2020high, goldt2019dynamics, d2020triple}. In this work, the phenomena we observe, such as the MtoG transition (left column of \Cref{fig:mem2gen}) and early learning dynamics (middle column of \Cref{fig:training_dynamic_sscd}), are consistent with observations reported in prior studies of models trained from scratch. This alignment suggests that distillation provides a tractable and controllable setting for understanding training in real-world.

\noindent \textbf{Training from scratch is increasingly close to distillation.}  
Modern diffusion models are often trained on large-scale datasets constructed from heterogeneous sources. Due to the scarcity of sufficiently high-quality real-world data, large-scale datasets such as LAION-5B \cite{schuhmann2022laion} inevitably contain a substantial fraction of synthetic data generated by foundation models. As a result, training from scratch increasingly resembles a teacher-student setting, where models implicitly learn from other pretrained models rather than purely from natural data.
 
Taken together, these facts motivate the use of $\PFD$ beyond explicit distillation settings, as a general tool for analyzing generalization in real-world large-scale diffusion models.

\section{Impact Statement}

This work advances the theoretical and empirical understanding of generalization in diffusion models by introducing Probability Flow Distance ($\PFD$). Improved understanding of generalization can help mitigate risks associated with memorization, including privacy leakage and copyright concerns. We do not anticipate direct negative societal impacts arising from this work; however, as with all advances in generative modeling, downstream applications should be developed and deployed responsibly.

\section*{Acknowledgment}

HJZ, ZJH, SYC, ZKZ, and QQ acknowledge support from NSF CCF-2212066, NSF CCF-
2212326, NSF IIS 2402950. The authors acknowledge valuable discussions with Prof. Saiprasad Ravishankar (MSU), Prof. Rongrong Wang (MSU), Prof. Jun Gao (U. Michigan), Mr. Xiang Li (U. Michigan), and Mr. Xiao Li (U. Michigan).

%% file: arxiv_deepthink/appendix/related_works.tex
\section{Related Works}\label{sec:related-work}

In this section, we briefly review related work on generalization metrics for diffusion models, discuss diffusion model generalizability, and cover the fundamentals of training diffusion models.

\subsection{Generalization Metrics for Diffusion Models}

Generalization metrics quantify the distance between the learned distribution and the underlying data distribution in diffusion models. To measure this distributional gap, theoretical works commonly employ metrics such as Kullback-Leibler (KL) divergence \cite{chen2023restoration, nie2024the, li2023generalization}, total variation (TV) \cite{chen2023the, li2024towards, li2024accelerating, yang2024leveraging, li2024sharp,liang2025low}, and Wasserstein distance \cite{gao2025convergence, debortoli2022convergence, chen2023improved, Gao2023WassersteinCG}. However, these metrics are practically inefficient for diffusion models. Practical metrics focus on various perspective, including negative log-likelihood (\texttt{NLL}) \cite{scoreSDE}, image generation quality: Fréchet inception distance ($\texttt{FID}$) \cite{heusel2017gans}, inception score ($\texttt{IS}$) \cite{salimans2016improved}, $\texttt{FD}_{\texttt{dinov2}}$ \cite{stein2024exposing}, maximum mean discrepancy ($\texttt{MMD}$) \cite{bińkowski2018demystifying}, CLIP maximum mean discrepancy (CMMD) \cite{jayasumana2024rethinking}; alignment: $\texttt{CLIPscore}$ \cite{hessel2021clipscore}, and $\texttt{precision, recall}$ \cite{sajjadi2018assessing, kynkaanniemi2019improved}. However, these practical metrics are not explicitly designed to evaluate the generalizability of diffusion models. Thus, there is a need for a generalization metric that are both theoretical grounded and practically efficient for diffusion models. To address this gap, we propose $\PFD$, a novel generalization metric that is theoretically proven to be a valid distributional distance and can be efficiently approximated by its empirical version using a polynomial number of samples. In practice, $\PFD$ requires fewer samples for estimation and is the only existing metric that explicitly quantifies generalization in diffusion models.

\subsection{Diffusion Model Generalizability}

Recent works have shown that diffusion models transition from memorization to generalization as the number of training samples increases \cite{yoon2023diffusion, zhang2023emergence}. With sufficient data, models trained with different architectures, loss functions, and even disjoint datasets can reproduce each other's outputs, indicating a strong convergence toward the underlying data distribution \cite{zhang2023emergence, kadkhodaie2023generalization}. To explain this strong generalization, \cite{kadkhodaie2023generalization} attribute it to the emergence of a geometric-adaptive harmonic basis, while others argue that generalization arises from interpolation across the data manifold \cite{aithal2024understanding, chen2025on}. In contrast, \cite{vastola2025generalization} explain generalization through inductive biases in the noise, whereas \cite{zhang2026generalization} attribute it to the emergence of a balanced representation space. Theoretical insights by \cite{li2023generalization} provide generalization bounds using KL-divergence under simplified models, \cite{ye2026provable} establish a dual theoretical separation explaining memorization and generalization.
More recent efforts focus on characterizing the learned noise-to-image mapping for generalized diffusion models, either through Gaussian parameterizations \cite{li2024understanding, wang2024unreasonable}, mixture of low rank Gaussian parameterizations \cite{wang2024diffusion, chen2024exploring}, multi-subspace multi-modal modeling \cite{yang2026multi} or patch-wise optimal score functions \cite{niedoba2024towards, kamb2024analytic}, and analyses based on statistical properties of image datasets \cite{lukoianov2025locality}. In parallel, generalization has also been studied through the lens of associative memory \cite{pham2025memorization, ambrogioni2024search} and sparse autoencoders \cite{tinaz2025emergence}. \cite{shi2025a} establishes a connection between generalization and model collapse.  However, despite these theoretical analyses and qualitative insights, prior work lacks a quantitative framework for measuring generalizability. In this paper, we propose $\PFD$, a metric that enables such quantitative evaluation. Using this measure, we uncover further insights into the generalization behavior of diffusion models, as discussed in \Cref{sec:gen-behavior}.

\subsection{Diffusion Model Distillation}
Recent work on diffusion distillation has demonstrated substantial efficiency gains through both lightweight student architectures \cite{chen2025snapgen, kim2024bk, hu2026snapgen++} and reduced sampling trajectories \cite{yin2024one, yin2024improved, salimans2022progressive, song2023consistency, meng2023distillation, sauer2024adversarial, xie2024em}. Representative approaches include progressive distillation \cite{salimans2022progressive, berthelot2023tract}, adversarial distillation \cite{jolicoeur-martineau2021adversarial, Xu_2024_CVPR, sauer2024adversarial}, and consistency-based distillation methods \cite{song2023consistency, wang2024phased}. Beyond architectural and sampling efficiency, Meng et al. \cite{meng2023distillation} distill the classifier-free guidance score \cite{ho2022classifier} to further accelerate guidance at inference time. Moreover, diffusion distillation has been successfully applied beyond efficiency improvements, including text-to-3D generation \cite{dreamfusion}, learning from corrupted data distributions \cite{zhang2025restoration}, and even achieving unexpected gains in generation quality \cite{ma2025diffusion}. In this paper, we adopt the most straightforward distillation paradigm, where the student is trained on data generated by a teacher model. Nevertheless, $\PFD$ is orthogonal to the choice of distillation strategy and can be naturally extended to more advanced techniques, such as progressive, adversarial, or consistency-based distillation frameworks, offering a flexible foundation for future extensions.

\subsection{Training Diffusion Models} \label{app:train_diffusion_model}

To enable sampling via the PF-ODE \eqref{eq:reve_ode_em}, we train a neural network $\bm s_{\bm\theta}(\bm x_t, t)$ to approximate the score function $\nabla \log p_{t}(\bm{x}_t)$ using denoising score matching loss \cite{scoreSDE}:  
\begin{align}\label{eq:em loss}
\min_{\bm \theta} \ell(\bm \theta) = \frac{1}{N}\sum_{i=1}^N \int_0^T \lambda_t \E_{\bm \epsilon\sim \mathcal{N}(\bm 0, T^2 \bm I_n)} & \left[\left\|\bm s_{\bm\theta}(\bm x^{(i)} + t \bm \epsilon, t)   + \bm \epsilon/t\right\|_2^2\right] \mathrm{d}t,  
\end{align}

$\lambda_t$ denotes a scalar weight for the loss at $t$. Given the learned score function, the corresponding noise-to-image mapping is:
\begin{align}\label{eq:n2i_integral_n}
    \bm \Phi_{p_{\bm \theta}}(\bm x_T) = \bm x_T - \int_{T}^{0}t \bm s_{\bm \theta}(\bm{x}_t, t) \mathrm{d}t.
\end{align}
Although alternative training objectives exist, such as predicting noise $\bm x_T$ \cite{ddpm}, clean image $\bm x_0$ \cite{EDM}, rectified flow $\bm x_T - \bm x_0$ \cite{RF} or other linear combinations of $\bm x_0$ and $\bm x_T$ \cite{salimans2022progressive}, prior works \cite{luo2022understanding, gao2025diffusion} have shown that it is still possible to recover an approximate score function $\bm s_{\bm \theta}(\bm{x}_t, t)$ from these methods. 

%% file: arxiv_deepthink/appendix/proof.tex
\section{Proof in Section \ref{sec:pf-distance}}\label{app:proof}

\begin{proof}[Proof of \Cref{thm:properties}]
It is trivial to show $\PFD(p, q) > 0$ for any $p \neq q$ and $\PFD(p, q) = \PFD(q, p)$, and thus we omit the proof. 

\begin{itemize}[leftmargin=*]
\item Proof of $p = q \Leftrightarrow \PFD(p, q) = 0:$
    \begin{itemize}
        \item ($\Rightarrow$) If $p = q$, $\nabla\log p_{t} \left(\bm x_t\right) = \nabla\log q_{t} \left(\bm x_t\right)$, thus:
        \begin{equation} \label{eq:ode_difference}
            \mathrm{d}\bm x_t = - t \left(\nabla \log p_{t}(\bm x_t) - \nabla \log q_{t}(\bm x_t)\right) \rm d t = 0
        \end{equation}
        Thus, $\bm \Phi_{p}(\bm x_T) - \bm \Phi_{q}(\bm x_T)$ is the solution of the ODE function \Cref{eq:ode_difference} with initial $\bm x_T = \bm 0$. Thus $\bm \Phi_{p}(\bm x_T) - \bm \Phi_{1}(\bm x_T) = \bm 0$ for all $\bm x_T$. Thus $\PFD(p, q) = 0$
        \item ($\Leftarrow$) If $\PFD(p, q) = 0$ and $\bm \Phi_{p}, \bm \Phi_{q}$ are continuous function w.r.t $\bm x_T$, then we have $\bm \Phi_{p} (\bm x_T) = \bm \Phi_{q} (\bm x_T)$ for all $\bm x_T$. If $\bm x_0 = \bm \Phi(\bm x_T)$, from the transformation of probability identities, we have:
        \begin{equation}\label{eq:trans_probability}
            p(\bm x_{0}) = \frac{\partial}{\partial [\bm x_{0}]_{1}} \ldots \frac{\partial}{\partial [\bm x_{0}]_{n}} \int_{\{\bm \epsilon| \bm \Phi(\bm \epsilon) \leq \bm x_{0}\}} p_{\mathcal{N}}(\bm \epsilon) \mathrm{d}^{n} \bm \epsilon,
        \end{equation}
    
        where $[\bm x_{0}]_{i}$ denotes the $i$-th element of $\bm x_{0}$, $\bm f(\bm \epsilon) \leq \bm x_{0}$ denotes the element wise less or equal. $p_{\mathcal{N}}(\cdot)$ is the probability density function (PDF) of Gaussian distribution $\mathcal{N}\left(\bm 0, T^2 \bm I_n\right)$. Thus, for all $\bm x_0$ we have:
        \begin{align}
        \begin{split}
            p(\bm x_{0}) - q(\bm x_{0}) &= \frac{\partial}{\partial [\bm x_{0}]_{1}} \ldots \frac{\partial}{\partial [\bm x_{0}]_{n}} \int_{\{\bm \epsilon| \bm \Phi_{p}(\bm \epsilon) \leq \bm x_{0}\}} p_{\mathcal{N}}(\bm \epsilon) \mathrm{d}^{n} \bm \epsilon \\
            & \ \ \ \ \ - \frac{\partial}{\partial [\bm x_{0}]_{1}} \ldots \frac{\partial}{\partial [\bm x_{0}]_{n}} \int_{\{\bm \epsilon| \bm \Phi_{q}(\bm \epsilon) \leq \bm x_{0}\}} p_{\mathcal{N}}(\bm \epsilon) \mathrm{d}^{n} \bm \epsilon, \\
            &= \frac{\partial}{\partial [\bm x_{0}]_{1}} \ldots \frac{\partial}{\partial [\bm x_{0}]_{n}} \int_{\{\bm \epsilon| \bm \Phi_{p}(\bm \epsilon) \leq \bm x_{0}\}} p_{\mathcal{N}}(\bm \epsilon) \mathrm{d}^{n} \bm \epsilon \\
            & \ \ \ \ \ - \frac{\partial}{\partial [\bm x_{0}]_{1}} \ldots \frac{\partial}{\partial [\bm x_{0}]_{n}} \int_{\{\bm \epsilon| \bm \Phi_{p}(\bm \epsilon) \leq \bm x_{0}\}} p_{\mathcal{N}}(\bm \epsilon) \mathrm{d}^{n} \bm \epsilon, \\ 
            &= 0,
        \end{split}
        \end{align}
        so $p = q$.
    \end{itemize}
\item Proof of $\PFD(p, q) \leq \PFD(p, p') + \PFD(p', q) $:
    \begin{align}
    \begin{split}
        &\PFD(p, q) \\
        = &\left( \mathbb E_{\bm x_T \sim \mathcal{N}(0, T^2  \bm I)} \left[\left\|  \bm \Phi_{p} \left(\bm x_T\right) -  \bm \Phi_{q}  \left(\bm x_T\right) \right\|^2_2\right] \right)^{1/2} \\
        \leq &\left( \mathbb E_{\bm x_T \sim \mathcal{N}(0, T^2  \bm I)} \left[\left(\left\|  \bm \Phi_{p} \left(\bm x_T\right) -  \bm \Phi_{p'}  \left(\bm x_T\right) \right\|_2 + \left\|  \bm \Phi_{p} \left(\bm x_T\right) -  \bm \Phi_{p'}  \left(\bm x_T\right) \right\|_2\right)^2\right]\right)^{1/2} \\
        \leq &\left( \mathbb E_{\bm x_T \sim \mathcal{N}(0, T^2  \bm I)} \left[\left\|  \bm \Phi_{p} \left(\bm x_T\right) -  \bm \Phi_{q}  \left(\bm x_T\right) \right\|^2_2\right] \right)^{1/2} \\
        & \  + \left( \mathbb E_{\bm x_T \sim \mathcal{N}(0, T^2  \bm I)} \left[\left\|  \bm \Phi_{p} \left(\bm x_T\right) -  \bm \Phi_{q}  \left(\bm x_T\right) \right\|^2_2\right] \right)^{1/2} \\
        = &\PFD(p, p') + \PFD(p', q) \\
    \end{split}
    \end{align}
\end{itemize}
\end{proof}

\begin{lemma}\label{lemma:lip_diff_n2i_mapping}
    Under \Cref{assump:lip_diff_score}, for all $\bm x_T \in \mathcal{N}(\bm 0, T^2\bm I_n)$, as $T \rightarrow \infty$, we have:
    \begin{equation}
        \left\| \bm \Phi_{p} \left(\bm x_T\right) - \bm \Phi_{q}  \left(\bm x_T\right) \right\|_2 \leq \exp\left(\dfrac{LT^2_\xi}{2}\right) \xi + \dfrac{\epsilon}{L}\left(\exp\left(\dfrac{LT^2_\xi}{2}\right) - 1\right),
    \end{equation}
    where $\xi$ is a numerical constant and a finite timestep $T_\xi$ depending only on $\xi$.
\end{lemma}

\begin{proof}[Proof of \Cref{lemma:lip_diff_n2i_mapping}]
    Let $\bm \phi_t, t \in [0, T]$ denotes the ODE trajectory:
    \begin{equation}
    \begin{aligned}
        &\bm \phi_t = \bm x^{p}_{t} - \bm x^{q}_{t}, \\
        &\bm x^{p}_{t} = \bm x_T - \int_{T}^{t}\tau\nabla_{\bm x} \log p_{\tau}(\bm x^{p}_{\tau}) \mathrm{d}\tau, \\
        &\bm x^{q}_{t} = \bm x_T - \int_{T}^{t}\tau\nabla_{\bm x} \log q_{\tau}(\bm x^{q}_{\tau}) \mathrm{d}\tau, \\
    \end{aligned}
    \end{equation}
    From the definition, $\bm \phi_0 =  \bm \Phi_{p} \left(\bm x_T\right) - \bm \Phi_{q}  \left(\bm x_T\right)$. Because $\lim_{T \rightarrow \infty} \bm \phi_t = \bm x_T - \bm x_T = \bm 0$, from the $\epsilon-\delta$ definition of the limit, given $\bm x_T$, and a constant $\xi$ , there exists a finite $T_\xi$ related to $\xi$ such that:
    \begin{equation}
        \left\|\bm \phi_t\right\|_2 \leq \xi \ \ \text{ for all } t \geq T_\xi.
    \end{equation}
    As $t \leq T_\xi$, we have:
    \begin{equation}\label{eq:ode_after_T0}
        \begin{aligned}
            &\frac{\mathrm{d} \bm \phi_t}{\mathrm{d} t} = - t \left(\nabla_{\bm x} \log p_{t}(\bm x^p_{t}) - \nabla_{\bm x} \log q_{t}(\bm x^q_{t})\right), \\
            &\left\|\bm \phi_{T_\xi}\right\|_2 \leq \xi.
        \end{aligned}
    \end{equation}
    Apply \Cref{assump:lip_diff_score} to \Cref{eq:ode_after_T0}, we could obtain the following integral inequality w.r.t $\left\|\bm \phi_t\right\|_2$:
    \begin{equation}    
        \begin{aligned}
            \frac{\mathrm{d} \left\|\bm \phi_t\right\|_2}{\mathrm{d} t}&\leq \left\|\frac{\mathrm{d} \bm \phi_t}{\mathrm{d} t}\right\|_2 \\
            &\leq t \left\|\nabla_{\bm x} \log p_{t}(\bm x^p_{t}) - \nabla_{\bm x} \log q_{t}(\bm x^q_{t})\right\|_2 \\
            &\leq t\left(\epsilon + L \left\|\bm \phi_t\right\|_2\right), \\
            \left\|\bm \phi_{T_\xi}\right\|_2 &\leq \xi, \ \ 0 \leq t \leq T_\xi,        \end{aligned}
    \end{equation}
    where the first inequality comes from the fact that $\dfrac{\mathrm{d} \left\|\bm \phi_t\right\|_2}{\mathrm{d} t} \leq \left\|\dfrac{\mathrm{d} \bm \phi_t}{\mathrm{d} t}\right\|_2$. From Grönwall's inequality \cite{coddington1955theory}, we could solve $\left\| \bm \Phi_{p} \left(\bm x_T\right) - \bm \Phi_{q}  \left(\bm x_T\right) \right\|_2 = \left\|\bm \phi_0\right\|_2 \leq \exp(\dfrac{LT^2_\xi}{2}) \xi + \dfrac{\epsilon}{L}\left(\exp(\dfrac{LT^2_\xi}{2}) - 1\right).$
\end{proof}

\begin{proof}[Proof of \Cref{thm:empirical_approximation}]Let $\bm X \coloneqq \left\| \bm \Phi_{p} \left(\bm x_T\right) - \bm \Phi_{q}  \left(\bm x_T\right) \right\|^2_2$. From \Cref{lemma:lip_diff_n2i_mapping}, 
\begin{align*}
    0 \leq \bm X \leq \kappa^2\left(L, \epsilon\right),
\end{align*}
with $\kappa\left(L, \epsilon\right) \coloneqq \exp\left(\dfrac{LT^2_\xi}{2}\right) \xi + \dfrac{\epsilon}{L}\left(\exp\left(\dfrac{LT^2_\xi}{2}\right) - 1\right)$. From Hoeffding's inequality \cite{hoeffding1994probability}, we have:
\begin{equation} \label{eq:hoeffding_inequality}
\P\left(\left|\E [\bm X] - \dfrac{1}{M} \sum\limits_{i = 1}^{M} \bm X_i\right| \geq \gamma\right) \leq 2 \exp\left(- \frac{2 M \gamma^2}{\kappa^4\left(L, \epsilon\right)}\right),
\end{equation}
with $M$ samples to achieve $\gamma$ accuracy. Thus, we could guarantee $\P\left(\left|\E [\bm X] - \dfrac{1}{M} \sum\limits_{i = 1}^{M} \bm X_i\right| \leq \gamma\right)$ with probability $\eta$, when:
\begin{equation}
    M \geq \frac{\kappa^4\left(L, \epsilon\right)}{2 \gamma^2}\log \frac{2}{\eta}.
\end{equation}

Because 
\begin{align}
    \left| \PFD(p, q) - \hat{\PFD}(p, q) \right| 
    =& \left|\sqrt{\E [\bm X]} - \sqrt{\dfrac{1}{M} \sum\limits_{i = 1}^{M} \bm X_i}\right| \\
    \leq& \sqrt{\left|\E [\bm X] - \dfrac{1}{M} \sum\limits_{i = 1}^{M} \bm X_i\right|}.
\end{align}

We could guarantee that $\P\left(\left| \PFD(p, q) - \hat{\PFD}(p, q) \right| \leq \gamma\right)$ with probability $\eta$, when:
\begin{equation}
    M \geq \frac{\kappa^4\left(L, \epsilon\right)}{2 \gamma^4}\log \frac{2}{\eta}.
\end{equation}

\end{proof}

\begin{exam}\label{exam:rp_score_gaussian}
 The Wasserstein-2 distance $\mathrm{W}_2(\cdot, \cdot)$ is the lower bound of the probability flow distance, i.e., 
 \begin{equation}\label{thm:relation_w2_pfd}
    \mathrm{W}_2(p, q) \leq \PFD(p, q),
\end{equation}
 Specifically, let $p$ and $q$ be multivariate Gaussian distribution $\mathcal{N}\left(\bm \mu_1, \bm \Sigma_1\right)$, $\mathcal{N}\left(\bm \mu_2, \bm \Sigma_2\right)$, respectively, where $\bm \mu_1, \bm \mu_2 \in \mathbb R^{n}$ and $\bm \Sigma_1, \bm \Sigma_2 \in \mathbb R^{n \times n}$. The $\PFD$ is given by
\begin{equation}
    \PFD\left(p, q\right) = \left(\left\|\bm \mu_1 - \bm \mu_2\right\|_2 + \left\|\bm \Sigma_1^{1/2} - \bm \Sigma_2^{1/2}\right\|_F\right)^{1/2} ,
\end{equation}
under this case, the equality $\PFD(p, q) = W_2(p, q)$ holds when $\bm \Sigma_1 \bm \Sigma_2 = \bm \Sigma_2 \bm \Sigma_1$.
\end{exam}

\begin{proof}[Proof of \Cref{exam:rp_score_gaussian}]

Proof of $W_2(p, q) \leq \PFD(p, q)$. From the definition of Wasserstein-2 distance:

\begin{align}
    W_2(p, q) = \inf_{\gamma\in \Gamma\left(p, q\right) } \left(\mathbb{E}_{\left(\bm x_{p}, \bm x_{q}\right) \sim \gamma } \left\|\bm x_{p} - \bm x_{q}\right\|_2^2\right)^{1/2},
\end{align}

where $\Gamma\left(p, q\right)$ is the set of all couplings of $p$ and $q$. As proofed by \cite{scoreSDE}, the noise-to-image mapping $\bm \Phi_{p}$ and $\bm \Phi_{q}$ pushes the Gaussian distribution $\mathcal{N}\left(\bm 0, T^2 \bm I_n\right)$ to the $p$ and $q$ distribution respectively. Thus we could find the coupling $\gamma_{\PFD} \coloneqq \left (\bm \Phi_{p}, \bm \Phi_{q} \right)_{\#}\mathcal{N}\left(\bm 0, T^2 \bm I_n\right)$, i.e., the pushforward of $\mathcal{N}\left(\bm 0, T^2 \bm I_n\right)$ by $\left(\bm \Phi_{p}, \bm \Phi_{q} \right)$, such that
\begin{equation}
    \PFD(p, q) = \left(\mathbb{E}_{\left(\bm x_{p}, \bm x_{q}\right) \sim \gamma_{\PFD} } \left\|\bm x_{p} - \bm x_{q}\right\|_2^2\right)^{1/2} \geq W_2(p, q)
\end{equation}

When distribution $p\left(\bm x\right)$ is Gaussian distribution $\mathcal{N}\left(\bm \mu, \bm \Sigma\right)$ with $\bm \mu \in \mathbb R^{n}$ and $\bm \Sigma \in \mathbb R^{n \times n}$. Thus $p_t\left(\bm x\right)$ is $\mathcal{N}\left(\bm \mu, \bm \Sigma + \sigma^2_t \bm I_{n}\right)$, thus the score function could be caluclated as,
\begin{align}\label{eq:score_gaussian}
    \nabla_{\bm x} \log p_t\left(\bm x\right) = \left( \bm \Sigma + t^2 \bm I_n\right)^{-1} \left(\bm \mu - \bm x\right).
\end{align}
By plugging in \Cref{eq:score_gaussian} to \Cref{eq:n2i_integral}, we could obtain the ODE equation w.r.t $\bm x$:
\begin{align}\label{eq:ode_gaussian}
    \mathrm{d}\bm{x} =  - t \left( \bm \Sigma + t^2 \bm I_n\right)^{-1} \left(\bm \mu - \bm x\right) \mathrm{d}t,.
\end{align}

The above ODE equation has a close form solution:
\begin{align}\label{eq:ode_sol_gaussian}
    \bm x_t =  \bm \mu + \bm U \diag\left(\left[\sqrt{\frac{\lambda_1 + t^2}{\lambda_1 + T^2}}, \ldots, \sqrt{\frac{\lambda_n + t^2}{\lambda_n + T^2}}\right]\right)\bm U^\top \left(\bm x_T - \bm \mu\right) 
\end{align}

where $\bm U, \lambda_k, k \in [n]$ are singlar value decomposition of $\bm \Sigma$, $\bm \Sigma = \bm U \diag\left(\left[\lambda_1, \ldots, \lambda_n\right]\right)\bm U^\top$. $\diag\left(\cdot\right)$ convert a vector in $\mathbb{R}^{n}$ into diagonal matrix $\mathbb{R}^{n \times n}$, and $\bm x_T \sim \mathcal {N}\left(\bm 0, T^2 \bm I_n\right)$. Let $\bm x_T = T \bm \epsilon$ with $\bm \epsilon \sim \mathcal {N}\left(\bm 0, \bm I_n\right)$. As $t = 0$ and $T \rightarrow \infty$, we have:
\begin{align}
    \bm x_t &= \left(\bm I_n - \bm U \diag\left(\left[\sqrt{\frac{\lambda_1 + t^2}{\lambda_1 + T^2}}, \ldots, \sqrt{\frac{\lambda_n + t^2}{\lambda_n + T^2}}\right]\right)\bm U^\top\right) \bm \mu, \\
    &\ \ \ \ \ \ \ + \bm U \diag\left(\left[T\sqrt{\frac{\lambda_1 + t^2}{\lambda_1 + T^2}}, \ldots, T\sqrt{\frac{\lambda_n + t^2}{\lambda_n + T^2}}\right]\right)\bm U^\top \bm x_T, \\
    &= \bm \mu + \bm U \diag\left(\left[\sqrt{\lambda_1}, \ldots, \sqrt{\lambda_n}\right]\right)\bm U^\top \bm x_T, \\
    &= \bm \mu + \bm \Sigma^{1/2} \bm x_T = \bm \Phi \left(\bm x_T\right).
\end{align}

Thus, plugging in \Cref{def:pfd}, we have:
\begin{align}
    \PFD\left(p, q\right) 
    &= \left(\mathbb E_{\bm x_T \sim \mathcal{N}(0, T^2\bm I)} \left[\left\| \bm \Phi_{1} \left(\bm x_T\right) - \bm \Phi_{2}  \left(\bm x_T\right) \right\|_2^2\right]\right)^{1/2} \\
    &= \left(\mathbb E_{\bm x_T \sim \mathcal{N}(0, T^2\bm I)}\left[\left\| \bm \mu_1 + \bm \Sigma_1^{1/2} \bm x_T -  \bm \mu_2 - \bm \Sigma_2^{1/2} \bm x_T\right\|_2^2\right]\right)^{1/2} \\
    &= \left(\left\|\bm \mu_1 - \bm \mu_2\right\|_2^2 + \left\|\bm \Sigma_1^{1/2} - \bm \Sigma_2^{1/2}\right\|^2_F\right)^{1/2} \\
    &= \left(\left\|\bm \mu_1 - \bm \mu_2\right\|_2^2 + \tr\left(\bm \Sigma_1 + \bm \Sigma_2 - 2 \bm \Sigma_1^{1/2}\bm \Sigma_2^{1/2}\right)\right)^{1/2}
\end{align}

From Wasserstein-2 distance for Gaussian distribution $p, q$ has close form solution: 
\begin{align}
    W_2\left(p, q\right) 
    &= \left(||\bm \mu_1 - \bm \mu_2||_2^2 + \tr\left(\bm \Sigma_1 + \bm \Sigma_2 - 2 \left(\bm \Sigma_1^{1/2} \bm \Sigma_2 \bm \Sigma_1^{1/2}\right)^{1/2}\right)\right)^{1/2}.
\end{align}
From \Cref{lemma:trace inqueality}, we have $W_2\left(p, q\right) \leq \PFD\left(p, q\right)$. And specifically, $W_2\left(p, q\right) = \PFD\left(p, q\right)$ when $\bm \Sigma_1 \bm \Sigma_2 = \bm \Sigma_2 \bm \Sigma_1$.
\end{proof}

\begin{lemma}\label{lemma:trace inqueality}
    Given two positive semi-definite matrix $\bm \Sigma_1, \bm \Sigma_2 \in \mathbb{R}^{n \times n}$, 
    \begin{equation}
        0 \leq \tr\left(\bm \Sigma_1^{1/2}\bm \Sigma_2^{1/2}\right) \leq \tr\left(\left(\bm \Sigma_1^{1/2}\bm \Sigma_2 \bm \Sigma_1^{1/2}\right)^{1/2}\right) .
    \end{equation}
\end{lemma}

\begin{proof}[Proof of \Cref{lemma:trace inqueality}]
    Because $\bm \Sigma_1, \bm \Sigma_2$ are positive semi-definite matrix, $\tr\left(\bm \Sigma_1^{1/2}\bm \Sigma_2^{1/2}\right) \geq 0$ and
    \begin{align}
        \tr\left(\left(\bm \Sigma_1^{1/2}\bm \Sigma_2 \bm \Sigma_1^{1/2}\right)^{1/2}\right) = \tr\left(\sqrt{\left(\bm \Sigma_1^{1/2}\bm \Sigma_2 ^{1/2}\right) \left(\bm \Sigma_1^{1/2}\bm \Sigma_2 ^{1/2}\right)^\top }\right) = \left\|\bm \Sigma_1^{1/2}\bm \Sigma_2 ^{1/2}\right\|_*,
    \end{align}
    where $||\cdot||_*$ is the nuclear norm (trace norm). From trace norm inequality (\cite{bhatia2013matrix} Chapter IV, Section 2), for a random matrix $\bm M$, $\tr\left(\bm M\right) \leq \left\|\bm M\right\|_*$. Thus, we have:
    \begin{equation}
        \tr\left(\bm \Sigma_1^{1/2}\bm \Sigma_2^{1/2}\right) \leq\left\|\bm \Sigma_1^{1/2}\bm \Sigma_2 ^{1/2}\right\|_*.
    \end{equation}

\end{proof}

%% file: arxiv_deepthink/appendix/experiments.tex
\section{Experiments}
\label{app:exp}

In this section, we provide experimental details and additional discussion of the main results.

\subsection{ Comparison of different metrics on synthetic datasets.}\label{app:compare_synthetic_metric}

\begin{figure}
\centering
\begin{subfigure}{0.35\linewidth}
    \centering
    \includegraphics[width=\linewidth]{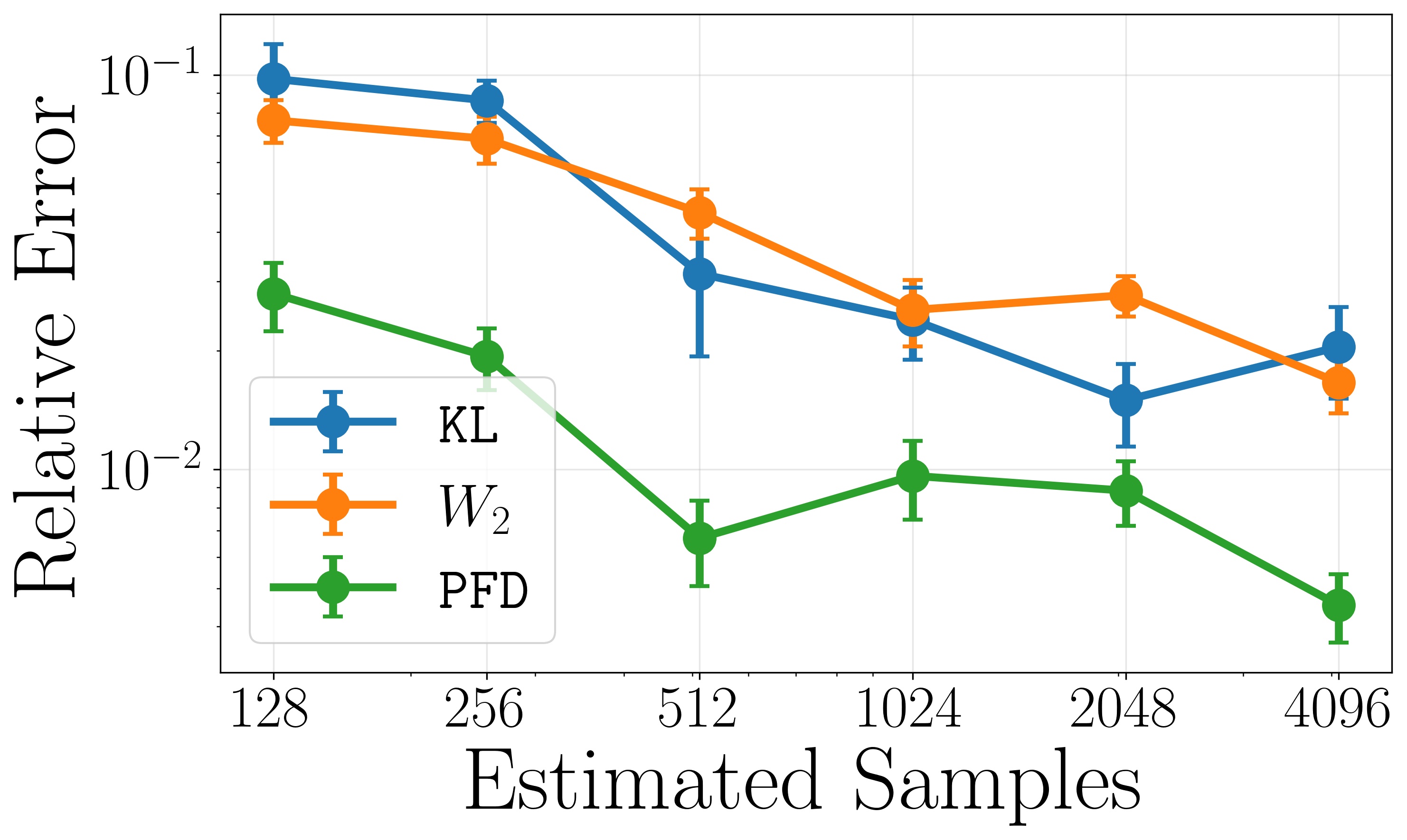}
    \caption{}
    \label{fig:compare_sample_efficiency}
\end{subfigure}
\begin{subfigure}{0.35\linewidth}
    \centering
    \includegraphics[width=\linewidth]{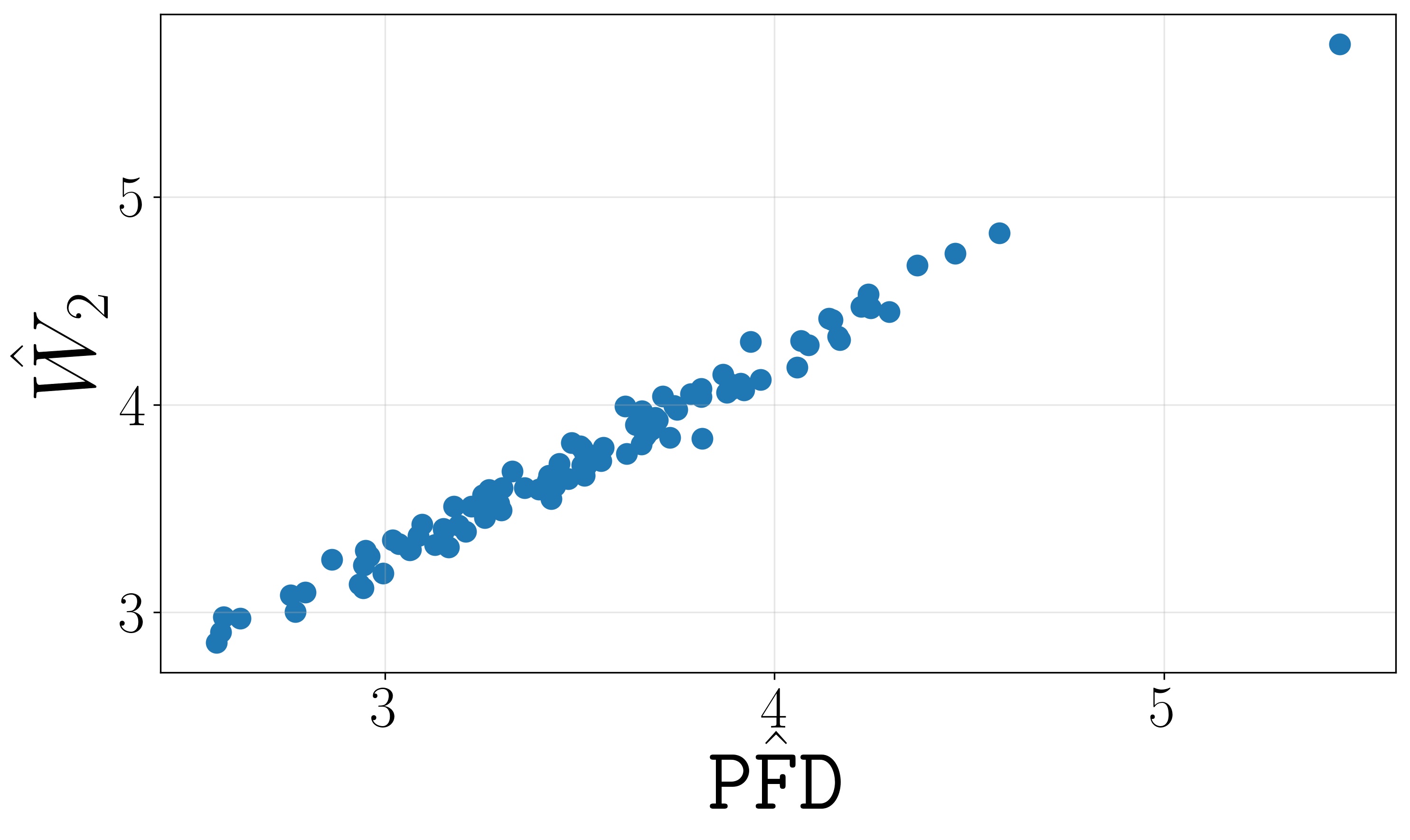}
    \caption{}
    \label{fig:compare_correlation}
\end{subfigure}
\caption{\textbf{Comparison of different metrics on synthetic datasets.} The figure illustrates (a) the sample efficiency of $\mathtt{KL}$, $W_2$, and $\PFD$ under a Gaussian distribution, and (b) the correlation between $W_2$ and $\PFD$ under a mixture of Gaussians.}
\label{fig:compare_synthetic}
\end{figure}

To support the claim that $\PFD$ is both sample-efficient and a meaningful distributional metric, we conduct numerical experiments on synthetic datasets, comparing $\PFD$ with the KL divergence ($\mathtt{KL}$) and the 2-Wasserstein distance ($W_2$). Specifically, in \Cref{fig:compare_sample_efficiency}, we evaluate the sample efficiency of these methods under multivariate Gaussian distributions. We set the dimension of the Gaussian distributions to $5$, randomly generate their means and variances, and repeat the experiment $10$ times. By varying the number of samples used for estimation from $128$ to $4096$, denoted by $M$, we report the relative error $|\mathcal{E} - \hat{\mathcal{E}}| / |\mathcal{E}|$, where $\mathcal{E}$ denotes the ground-truth value and $\hat{\mathcal{E}}$ its empirical estimate.
In \Cref{fig:compare_correlation}, we examine the correlation between estimates of $W_2$ and $\PFD$ on mixtures of Gaussian distributions. We consider mixtures of $5$ Gaussian components, each with $5$ dimensions. Both metrics are estimated using $M = 4096$ samples, and the experiment is repeated $100$ times to generate the plotted results.

As shown in \Cref{fig:compare_sample_efficiency}, $\PFD$ exhibits significantly better sampling efficiency than both $\mathtt{KL}$ and $W_2$. With $M=4096$ samples, $\PFD$ achieves a relative error of approximately $4 \times 10^{-2}$, whereas $\mathtt{KL}$ and $W_2$ incur errors on the order of $2 \times 10^{-1}$. Moreover, the computational complexity of $\PFD$ is $O(M)$, compared to $O(M^2)$ for $W_2$. Consequently, under the same sampling budget, $\PFD$ is substantially more efficient in both estimation accuracy and computational cost.

As illustrated in \Cref{fig:compare_correlation}, when measuring the distance between two mixture-of-Gaussian distributions, the estimated $\hat{\PFD}$ exhibits a strong linear correlation with $\hat{W}_2$ (with correlation coefficient 0.992). This result indicates that $\PFD$ captures meaningful distributional distance, which align with 2-Wasserstein distance, even for more complex, multimodal distributions.

\subsection{Network Architecture Details} \label{app:architecture}

In this subsection, we provide details of the U-Net architectures, as summarized in \Cref{tab:architecture}. The U-Net follows an encoder-decoder design, where the encoder comprises multiple encoder blocks. The column "\textbf{Dimensions for encoder blocks}" indicates the feature dimensions of each encoder block, while "\textbf{Number of residual blocks}" specifies how many residual blocks are used within each encoder block. The decoder is symmetric to the encoder. For further architectural details, please refer to \cite{edm_code}. By varying the encoder block dimensions and the number of residual blocks, we scale the U-Net model from 0.9M to 55.7M parameters.

\begin{table}[]
    \centering
    \caption{\textbf{U-Net architectures details.}}
    \resizebox{\textwidth}{!}{%
    \begin{tabular}{lccc}
    \toprule
    \textbf{Name} & \textbf{Dimensions for encoder blocks} & \textbf{Number of residual blocks} & \textbf{Number of parameters $|\bm \theta|$} \\
    \midrule
    U-Net-1 & [32, 32, 32] & 4 & 0.9M \\
    U-Net-2 & [64, 64, 64] & 4 & 3.5M \\
    U-Net-3 & [96, 96, 96] & 4 & 7.9M \\
    U-Net-4 & [128, 128, 128] & 4 & 14.0M \\
    U-Net-5 & [80, 160, 160] & 4 & 17.1M \\
    U-Net-6 & [160, 160, 160] & 3 & 17.8M \\
    U-Net-7 & [160, 160, 160] & 4 & 21.8M \\
    U-Net-8 & [192, 192, 192] & 4 & 31.3M \\
    U-Net-9 & [224, 224, 224] & 4 & 42.7M \\
    U-Net-10 & [256, 256, 256] & 4 & 55.7M \\
    \bottomrule
    \end{tabular}
    }
    \label{tab:architecture}
\end{table}

\subsection{Evaluation Protocol} \label{app:evaluate_protocol_setting}

In this subsection, we provide details of the evaluation protocol introduced in \Cref{sec:gen-eval}, as well as the comparison between the synthetic dataset from the teacher model and the real dataset. 

\noindent \textbf{Experiment settings for evaluation protocol.} The teacher model $\bm \theta_t$ and the student model $\bm \theta$ share a similar U-Net architecture \cite{unet} with different numbers of parameters, as introduced in \Cref{app:architecture}. The teacher model, with UNet-10 architecture, is trained on the CIFAR-10 dataset \cite{krizhevsky2009learning} using the EDM noise scheduler \cite{EDM}, with a batch size of 128 for 1,000 epochs. The student model \footnote{The architecture of the student model varies across experiments and will be described in detail for each specific case.} is trained using the variance-preserving (VP) noise scheduler \cite{ddpm}, under the same training hyperparameters. We use one A40 GPU with 48 GB video random access memory (VRAM) for all experiments. We generated three subsets of initial noise $\{\bm{x}^{(i)}_{\texttt{train}, T}\}_{i=1}^{N}, \{\bm{x}^{(i)}_{\texttt{gen}, T}\}_{i=1}^{M}, \{\bm{x}^{(i)}_{\texttt{test}, T}\}_{i=1}^{M} \overset{\mathrm{iid}}{\sim} \mathcal{N}(0, T^2 \bm{I}_n)$. The training and test datasets are produced using the teacher model: 
\begin{equation*}
  \mathcal{D} \coloneqq \{\bm x^{(i)}_{\texttt{train}} \}_{i = 1}^{N} = \{\bm \Phi_{p_{\bm \theta_t}}(\bm{x}^{(i)}_{\texttt{train}, T}) \}_{i = 1}^{N},\ \  \mathcal{D}_{\texttt{test}} \coloneqq \{\bm x^{(i)}_{\texttt{test}} \}_{i = 1}^{M} = \{\bm \Phi_{p_{\bm \theta_t}}(\bm{x}^{(i)}_{\texttt{test}, T}) \}_{i = 1}^{M}.
\end{equation*}
To evaluate the student model, we generate an evaluation dataset from itself: 
\begin{equation*}
   \mathcal{D}_{\texttt{gen}} \coloneqq \{\bm x^{(i)}_{\texttt{gen}} \}_{i = 1}^{M} = \{\bm \Phi_{p_{\bm \theta}}(\bm{x}^{(i)}_{\texttt{gen}, T}) \}_{i = 1}^{M}.
\end{equation*}
All samples are generated using the second-order Heun solver \cite{EDM} with 18 sampling steps. We vary the number of training samples $N$ from $2^6$ to $2^{16}$ in powers of two. $M$ is set to 50,000 for the experiments in \Cref{app:compare_metric}, and 10,000 for the rest.

\textbf{Experiment settings for validating the synthetic dataset with real real-world dataset.} We evaluate $\texttt{FID}$ and $\PFDm$ for diffusion models with UNet-4 architecture, trained separately on the synthetic dataset $\mathcal{D}$ and CIFAR-10 training dataset. We keep the number of training dataset $N$ the same for these two settings, ranging from $2^6$ to $2^{15}$, with a power of 2. Then we evaluate the $\mathtt{FID}$ between $\mathcal{D}_{\texttt{gen}}$ and $\mathcal{D}_{\texttt{test}}$ (CIFAR-10 test dataset) for the synthetic (real-world) setting, with $M= 10000$. To evaluate $\PFDm$, we use the initial noise $\{\bm x^{(i)}_{\texttt{gen}} \}_{i = 1}^{M}$. 

\subsection{Comparison with Practical Metrics for Generalization Evaluation}
\label{app:compare_metric}

\begin{figure}[t]
\begin{center}
    \includegraphics[width = \linewidth]{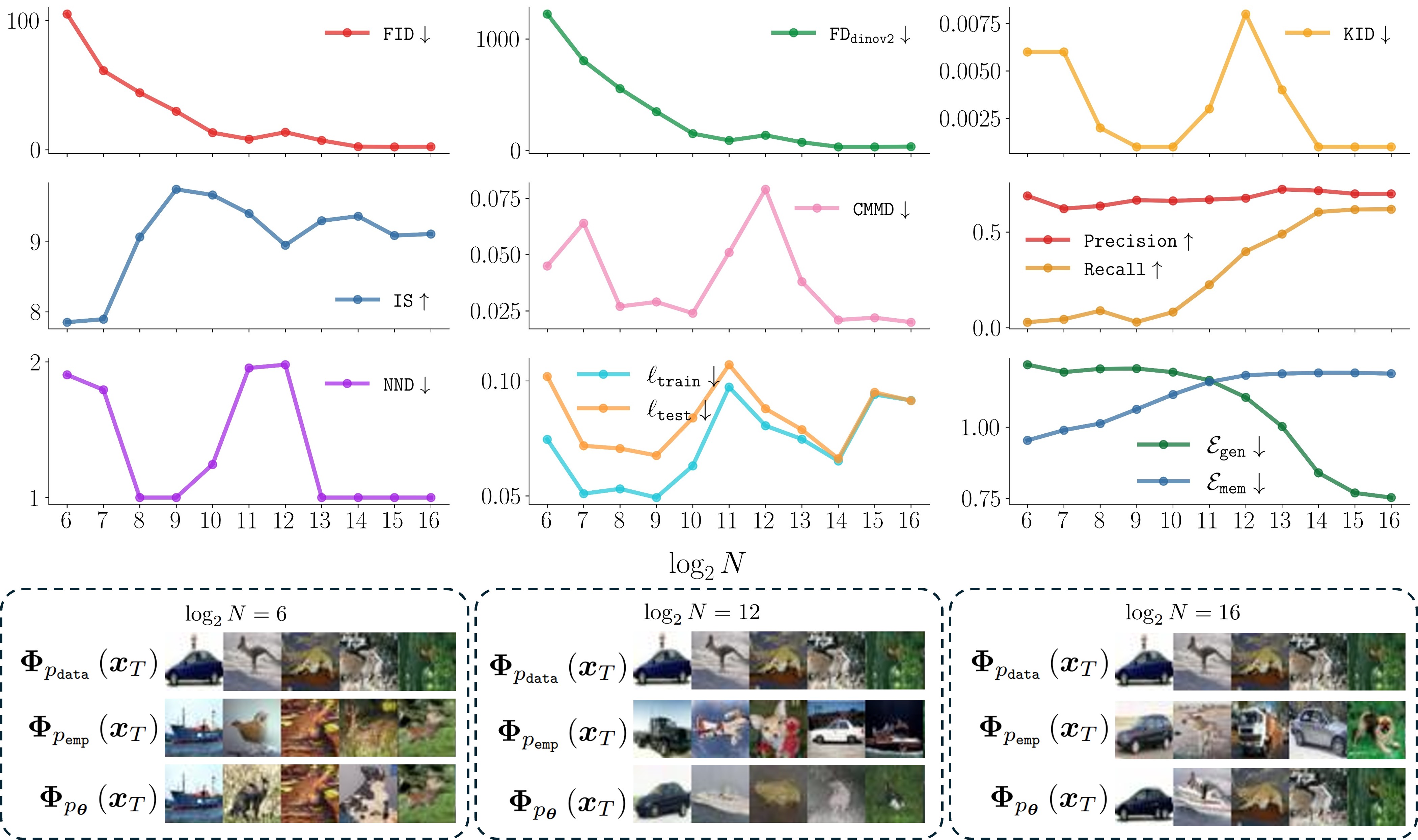}
    \caption{\textbf{Comparison of practical metrics on the MtoG transition.} The top figure plots multiple evaluation metrics as functions of $\log_2 N$. The bottom figure visualizes the generation under three numbers of training samples ($2^6, 2^{12}, 2^{16}$). For each setting, the figure shows generations from the underlying distribution (top row), empirical data distribution (middle row), and the learned distribution from the diffusion model (bottom row). Each column corresponds to the same initial noise.}
    \label{fig:metric-comparison}
\end{center}
\end{figure}

In this subsection, we expand upon the experiment presented in \Cref{sec:gen-eval}, which compares our proposed metric with practical metrics for evaluating generalization. We compare $\PFDg$ and $\PFDm$ with well-used generative model metrics, inluding $\texttt{FID}, \texttt{FD}_{\texttt{DINOv2}}, \texttt{KID}, \texttt{CMMD}, \texttt{Precision}, \texttt{Recall}, \texttt{NND}, \texttt{IS}$.  We also including the training and testing loss $\ell_{\texttt{train}}, \ell_{\texttt{test}}$ (\Cref{eq:em loss}) as comparison. We evaluating their ability  in capturing the MtoG transition, under the evaluation protocol proposed in \Cref{sec:gen-eval}.

\begin{table}[htbp]
\centering
\caption{Datasets used to evaluate each metric.}
\begin{tabular}{ll}
\toprule
\textbf{Metric} & \textbf{Dataset(s)} \\
\midrule
\makecell[l]{\texttt{FID}, $\texttt{FD}_{\texttt{DINOv2}}$, \texttt{KID}, \texttt{CMMD},\\ \texttt{Precision}, \texttt{Recall}, \texttt{NND}} & $\mathcal{D}_{\texttt{gen}}$ vs. $\mathcal{D}_{\texttt{test}}$ \\
$\texttt{FID}_{\texttt{train}}$, $\texttt{FD}_{\texttt{DINOv2,train}}$ & $\mathcal{D}$ vs. $\mathcal{D}_{\texttt{test}}$ \\
\texttt{IS} & $\mathcal{D}_{\texttt{gen}}$ \\
$\ell_{\texttt{train}}$ & $\mathcal{D}$ \\
$\ell_{\texttt{test}}$ & $\mathcal{D}_{\texttt{test}}$ \\
$\PFDm$, $\PFDg$ & $\{\bm{x}^{(i)}_{\texttt{gen}, T}\}_{i=1}^{M}$ \\
\bottomrule
\end{tabular}
\label{tab:metric_dataset}
\end{table}

We use UNet-10 for the student model in this experiment. We summarized datasets used by these metrics in \Cref{tab:metric_dataset}. Results are shown in \Cref{fig:metric-comparison}, summarized into one sentence, only $\PFDg$ and $\PFDm$ could quantitatively capture this transition. We include detailed discussions below:

\textbf{Results discussions.} \Cref{fig:metric-comparison} (bottom) is consistent with prior empirical observations \cite{zhang2023emergence, yoon2023diffusion}: In the memorization regimes ($N = 2^6$), $p_{\bm \theta}$ tends to memorize the empirical distribution $\pemp$, resulting in similar generation between $\bm \Phi_{\pemp}(\bm{x}_T)$ and $\bm \Phi_{p_{\bm \theta}}(\bm{x}_T)$; in the transition regime ($N = 2^{12}$), the model lacks sufficient capacity to memorize and the sample complexity is inadequate for generalization, leading to poor-quality generations $\bm \Phi_{p_{\bm \theta}}(\bm{x}_T)$; in the generalization regimes ($N = 2^{16}$), $p_{\bm \theta}$ captures the underlying distribution $\pdata$, and the generations $\bm \Phi_{\pdata}(\bm{x}_T)$ and $\bm \Phi_{p_{\bm \theta}}(\bm{x}_T)$ are closely aligned.

As shown in \Cref{fig:metric-comparison} (top), when $N$ increases, $\PFDm$ consistently increases and $\PFDg$ consistently decreases. This aligns with our intuition: as sample complexity grows, models tend to generalize and memorize less. In contrast, all other metrics fail to capture this transition effectively. The reasons can be summarized as follows:
\begin{itemize}[leftmargin=*]
\item \textbf{\texttt{FID}, $\texttt{FD}_{\texttt{DINOv2}}$, $\texttt{KID}$, $\texttt{IS}$, and $\texttt{CMMD}$ are sensitive to generation quality.} Image quality metrics, including \texttt{FID}, $\texttt{FD}_{\texttt{DINOv2}}$, $\texttt{KID}$, $\texttt{IS}$, and $\texttt{CMMD}$, show degradation in performance at $N = 2^{12}$. This drop is primarily due to degraded visual quality in the generated samples, as visualize in \Cref{fig:metric-comparison} (bottom-middle). However, at this sample complexity, the generated data still captures low-level features such as colors and structures from the underlying distribution. This is evident from the visual similarity between $\bm \Phi_{\pdata}(\bm{x}_T)$ and $\bm \Phi_{p_{\bm \theta}}(\bm{x}_T)$, suggesting the model have some generalizability. In comparison, only $\PFDg$ decreases consistently around $N = 2^{12}$, indicating it captures generalizability better than others despite visual degradation.

\item \textbf{\texttt{FID}, $\texttt{FD}_{\texttt{DINOv2}}$ and $\texttt{Recall}$ are sensitive to diversity.} The monotonic trends for \texttt{FID}, $\texttt{FD}_{\texttt{DINOv2}}$ and $\texttt{Recall}$ are due to their sensitivity to the diversity of $\mathcal{D}_{\texttt{gen}}$, rather than their ability to measure generalizability. At small $N$, the model memorizes the training samples, resulting in $\mathcal{D}_{\texttt{gen}}$ closely resembling $\mathcal{D}$ and exhibiting significantly lower diversity than $\mathcal{D}_{\texttt{test}}$, since $N \ll M$. Under these conditions, \texttt{FID}, $\texttt{FD}_{\texttt{DINOv2}}$ are large because they are biased towards the diversity of the evaluation samples (as proved in \cite{chong2020effectively}). Meanwhile, $\texttt{Recall}$ is low because the the support of $\mathcal{D}_{\texttt{test}}$ is limited, reducing the probability that samples drawn from $\mathcal{D}_{\texttt{gen}}$ lie within the support of $\mathcal{D}_{\texttt{test}}$.
In contrast, $\PFDg$ measures generalizability by directly quantifying the distance between the generation from the learned distribution and the underlying distribution and is less affected by the diversity of the generated samples.

\item \textbf{$\texttt{NND}$ and $\ell$ fail to capture the generalizability.} The $\texttt{NND}$, originally designed for assessing the generalization of GANs, is sensitive to image quality and increases during the transition regime. Additionally, it produces identical values across a wide range of sample sizes (e.g., $N = 2^8, 2^9, 2^{13}, 2^{14}, 2^{15}, 2^{16}$), making it unreliable for evaluating generalization in diffusion models. Similarly, neither the training loss $\ell_{\texttt{train}}$ nor the test loss $\ell_{\texttt{test}}$ exhibits a consistent decreasing trend as $N$ increases, indicating that these losses do not directly reflect either memorization or generalization. While the loss gap  $\ell_{\texttt{test}} - \ell_{\texttt{train}}$ does tend to decrease with larger $N$, it cannot serve as a robust generalization metric either. This is because even a randomly initialized model $\bm \theta$ can exhibit a small loss gap.
\end{itemize}

In conclusion, $\PFDm$ and $\PFDg$ are the only metrics that could capture the MtoG transition for diffusion models. They evaluate the generalization (memorization) by directly measuring the distance between the learned distribution by the diffusion model and the underlying (empirical) distribution. Unlike other metrics, they are less affected by the quality or diversity of the evaluating samples.

\subsection{Scaling Behaviors of the MtoG Transition} 
\label{app:scaling_behavior}

In this subsection, we provide detailed experimental settings for \Cref{sec:scaling_behavior}, along with additional experiments to further investigate the MtoG transition across more architectures (e.g., Transformer-based models \cite{bao2023all}). We also investigate the scaling behavior of the MtoG transition under the DINOv2 descriptor.

\begin{figure}[t]
\begin{center}
    \includegraphics[width = \linewidth]{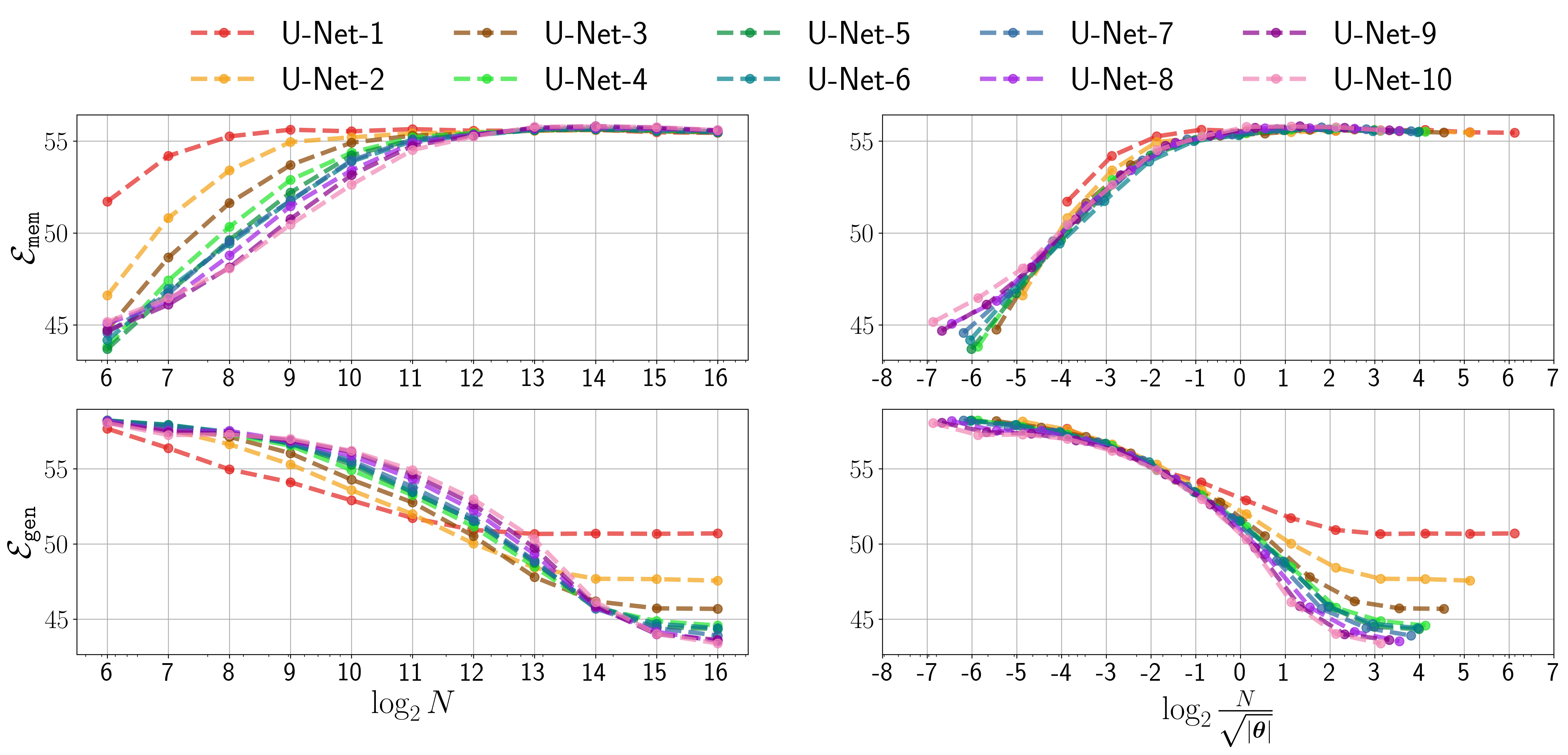}
    \caption{\textbf{Scaling behavior in the MtoG transition under DINOv2 descriptor.} $\PFDm$ and $\PFDg$ plotted against $\log_2(N)$ for a range of U-Net architectures (U-Net-1 to U-Net-10). Right: the same metrics plotted against $\log_2(N/\sqrt{|\bm \theta|})$, where $|\bm \theta|$ is the number of model parameters.}
    \label{fig:mem2gen_dinov2}
\end{center}
\end{figure} 

\noindent \textbf{Experiment settings.} The detailed architectures of the student models, from U-Net-1 to U-Net-10, are provided in \Cref{app:architecture}, with model sizes ranging from 0.9M to 55.7M parameters. We scale up the architectures by increasing the dimensionality of the encoder blocks and the number of residual blocks. For the ImageNet experiments, we adopt the U-Net architectures proposed in \cite{karras2024analyzing}, referred to as U-Net-11 and U-Net-12. These models contain 124.2M and 295.9M parameters, respectively.

\begin{figure}[t]
\begin{center}
    \includegraphics[width = \linewidth]{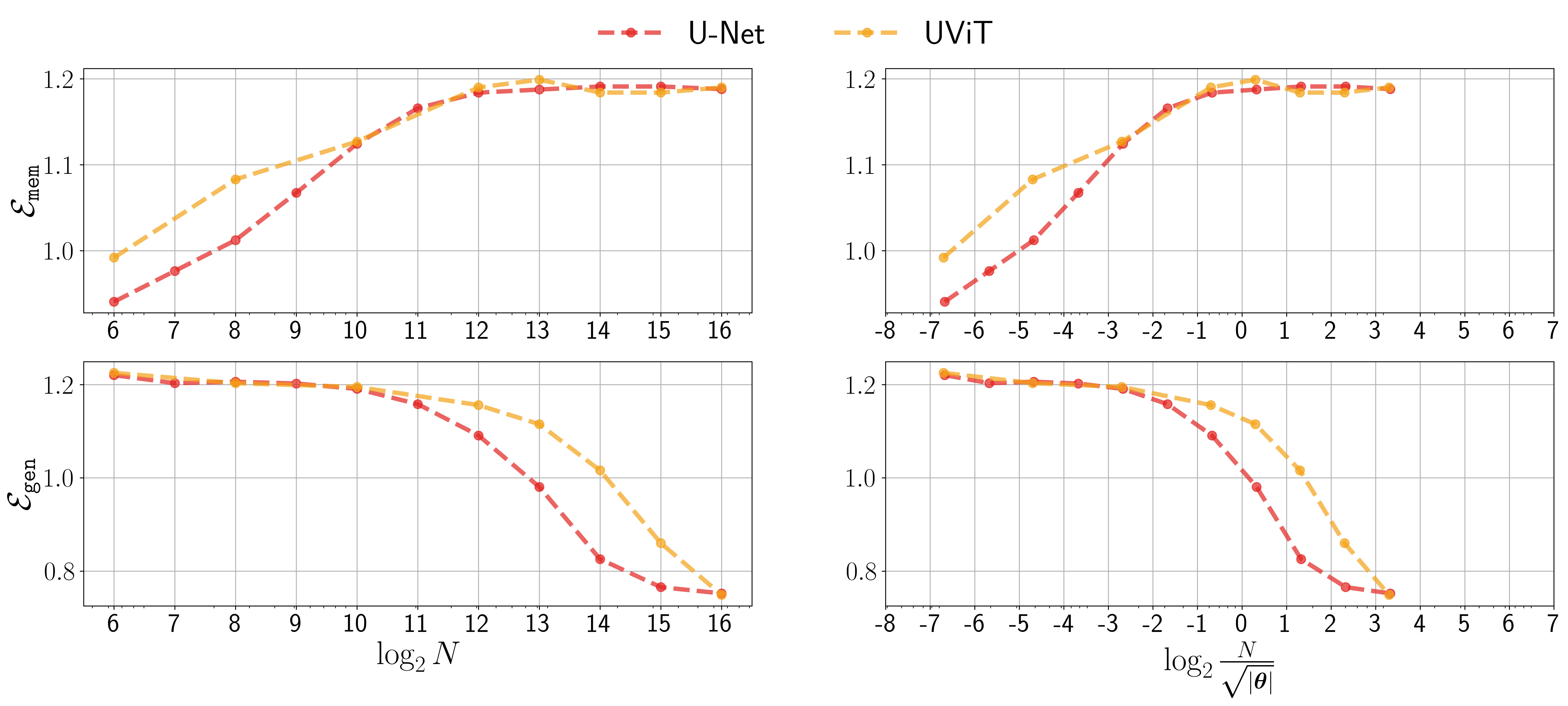}
    \caption{\textbf{Comparison of scaling behavior between UNet and Transformer architectures in the MtoG transition.} $\PFDm$ and $\PFDg$ plotted against $\log_2(N)$ for U-Net architecture (U-Net-9) and UViT architecture. Right: the same metrics plotted against $\log_2(N/\sqrt{|\bm \theta|})$, where $|\bm \theta|$ is the number of model parameters.}
    \label{fig:mem2gen_transformer}
\end{center}
\end{figure}

\noindent \textbf{MtoG transition between U-Net and transformer architecture.} To further investigate the impact of network architecture, we compare the U-Net architecture with the transformer-based UViT \cite{bao2023all}. Specifically, we use the U-Net-9 from \Cref{tab:architecture}, containing 42.7M parameters, and design the UViT model with comparable parameters of 44.2M. Both models are trained for 1000 epochs. Using the same experimental setup described in \Cref{sec:scaling_behavior}, we plot the MtoG transition curves for both U-Net and UViT, as shown in \Cref{fig:mem2gen_transformer}.

As illustrated in \Cref{fig:mem2gen_transformer}, with a similar number of parameters and the same training data sizes, UViT exhibits a higher $\PFDm$ in the memorization regime ($2^6 \leq N \leq 2^{10}$) and a higher $\PFDg$ in the generalization regime ($2^{11} \leq N \leq 2^{15}$), suggesting a lower model capacity compared to U-Net under these conditions. However, when provided with sufficient training data ($N = 2^{16}$), UViT achieves a lower $\PFDg$, demonstrating better generalization performance. This observation is consistent with prior findings on transformer architectures in classification tasks: transformer-based models, lacking the inductive biases inherent to CNNs, tend to generalize poorly when trained on limited data \cite{dosovitskiy2021an}.

\noindent \textbf{Scaling behavior of the MtoG transition under the DINOv2 descriptor.} The scaling behavior under the DINOv2 descriptor is shown in \Cref{fig:mem2gen_dinov2}. Both $\PFDm$ and $\PFDg$ exhibit trends consistent with those observed under the SSCD descriptor (see \Cref{fig:mem2gen}). The only difference is that, under the DINOv2 descriptor, models with varying parameter sizes show greater differentiation in the generalization regime compared to those under the SSCD descriptor. Further discussion on this can be found in the ablation study on image descriptors in \Cref{app:diff_embedding_feature}.

\subsection{Early Learning and Double Descent in Learning Dynamics}
\label{app:training_dynamics}

\begin{figure}[t]
\begin{center}
    \includegraphics[width = \linewidth]{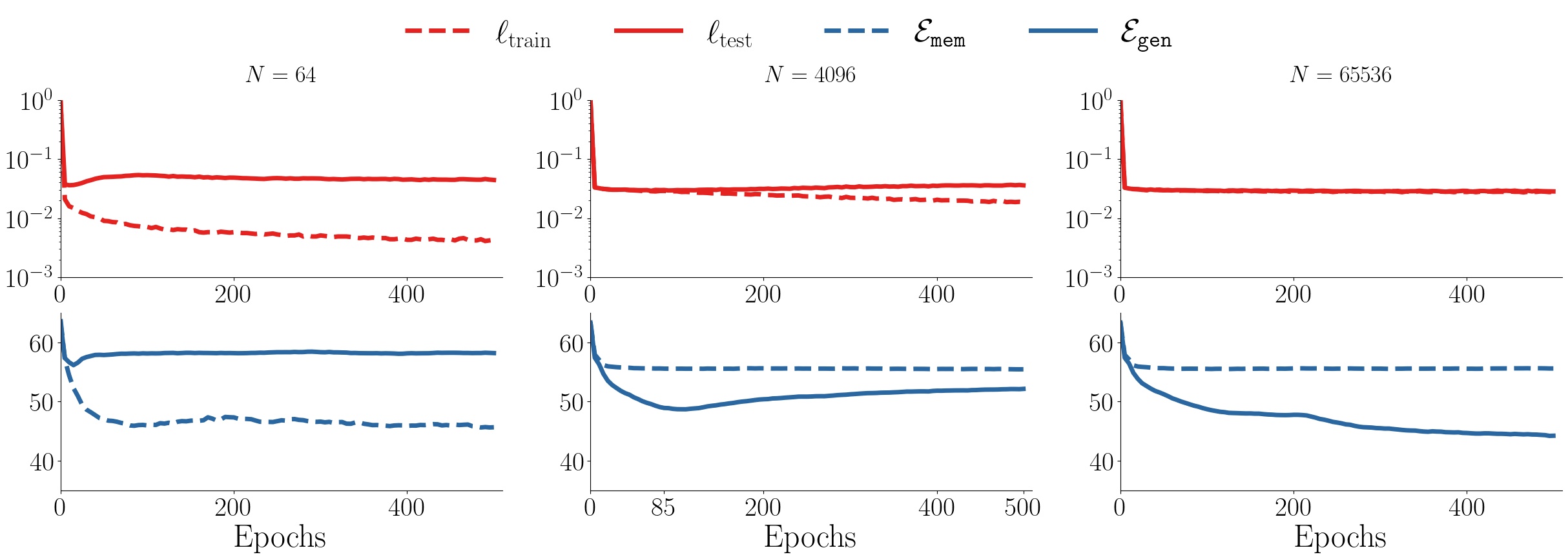}
    \caption{\textbf{Training dynamics of diffusion models under DINOv2 descriptor in different regimes.} The figure plots $\PFDm, \PFDg, \ell_{\texttt{train}}, \ell_{\texttt{test}}$ over training epochs for different different dataset sizes: $N = 2^6$ (left), $2^{12}$ (middle), $2^{16}$ (right). }\label{fig:training_dynamic_dinov2}

\end{center}
\end{figure} 

In this subsection, we build on the discussion from \Cref{sec:training_dynamics}. In \Cref{fig:training_dynamic_sscd}, we evaluate $\ell_{\texttt{train}}$ and $\ell_{\texttt{test}}$ across the three training regimes. Notably, the gap $\ell_{\texttt{test}} - \ell_{\texttt{train}}$ emerges as a practical heuristic for identifying the training regime: In the memorization regime, the gap increases steadily with training; In the transition regime, the gap remains near zero during early training (when generalization improves) and increases for further training (when generalization degrades); in the generalization regime, the gap remains close to zero throughout training. While $\ell_{\texttt{test}} - \ell_{\texttt{train}}$ is not a strict measure of generalization, it proves to be a useful empirical indicator of training regimes for diffusion models. Practically, by setting aside a test dataset to estimate this gap, we can more effectively identify the training regime for diffusion models.

\noindent \textbf{Training dynamics of diffusion models under the DINOv2 descriptor.} The training dynamics under the DINOv2 descriptor are shown in \Cref{fig:training_dynamic_dinov2}. Both $\PFDm$ and $\PFDg$ exhibit trends consistent with those observed under the SSCD descriptor for $N = 64$ and $N = 4096$ (see \Cref{fig:training_dynamic_sscd}). For $N = 65536$, $\PFDg$ still displays a double descent pattern under the DINOv2 descriptor; however, instead of a rise between the two drops, the curve remains relatively flat.

\subsection{Bias-Variance Decomposition of Generalization Error}
\label{app:bias-var}

To approximate $\overline{\bm \Psi \circ \bm \Phi}_{p_{\bm \theta}}(\cdot)$, we independently sample two training datasets, $\mathcal{D}_1$ and $\mathcal{D}_2$, for each specified number of training samples $N$. We then train two student models, $\bm \theta(\mathcal{D}_1)$ and $\bm \theta(\mathcal{D}_2)$, using these datasets. The quantity $\overline{\bm \Psi \circ \bm \Phi}_{p_{\bm \theta}}(\cdot)$ is approximated as follows:

\begin{equation}
    \overline{\bm \Psi \circ \bm \Phi}_{p_{\bm \theta}}  (\cdot) \approx \frac{1}{2}(\bm \Psi \circ \bm \Phi_{p_{\bm \theta\left(\mathcal{D}_1\right)}}(\cdot) + \bm \Psi \circ \bm \Phi_{p_{\bm \theta\left(\mathcal{D}_2\right)}}(\cdot) ).
\end{equation}

\noindent \textbf{Bias-Variance Decomposition of Generalization Error under the DINOv2 Descriptor.} The bias-variance decomposition under the DINOv2 descriptor is shown in \Cref{fig:bias-variance-decomposition-dinov2}. Overall, the results are consistent with those observed under the SSCD descriptor, with two differences: (1) for $N = 65536$, $\PFDg$ does not exhibit a U-shaped curve under the DINOv2 descriptor; and (2) $\PFDb$ for U-Net-1 and U-Net-2 does not decrease monotonically, instead, it first decreases and then increases.

\begin{figure}[t]
\begin{center}
    \begin{subfigure}[t]{0.38\linewidth}
        \includegraphics[width=\linewidth]{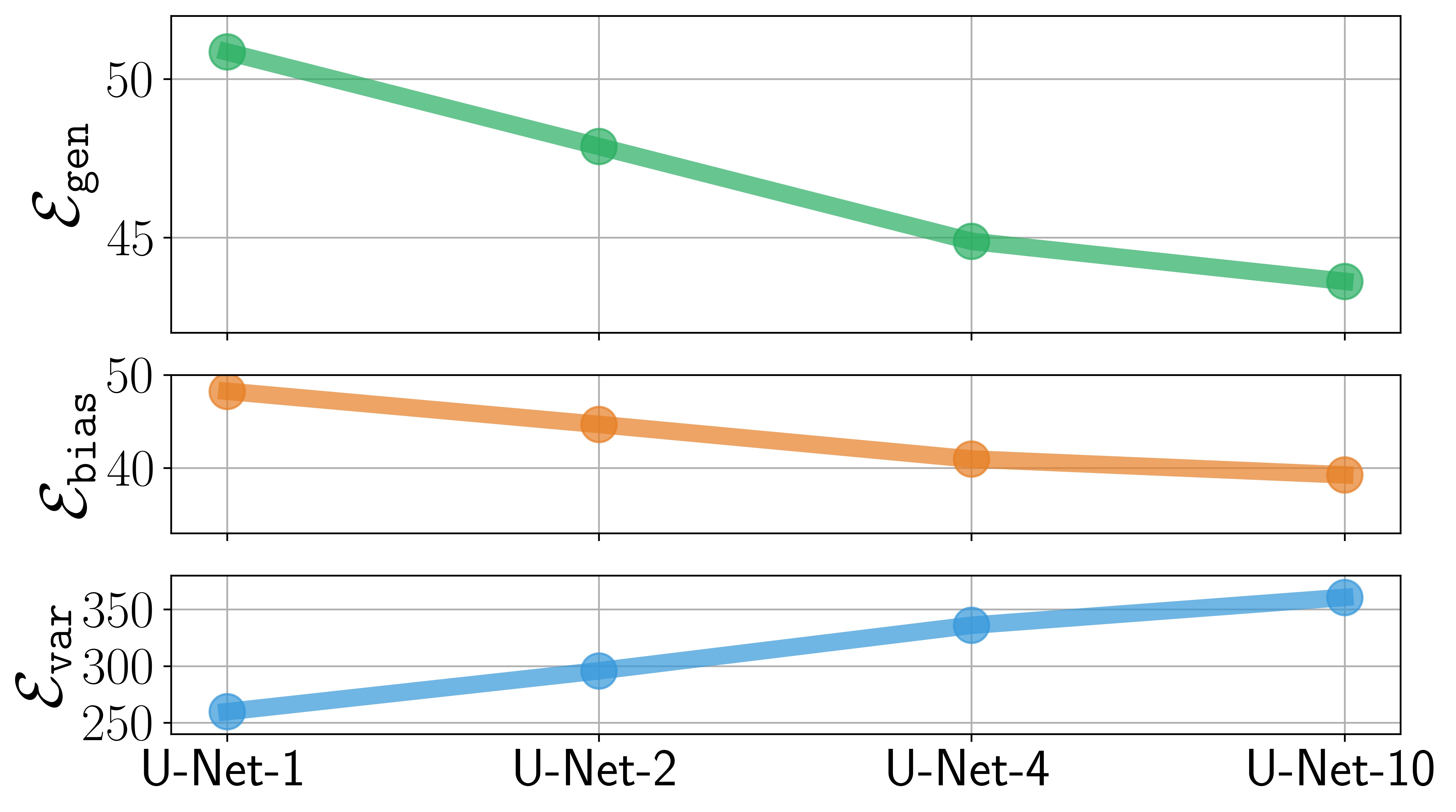}
        \caption{}
    \end{subfigure}
    \begin{subfigure}[t]{0.6\linewidth}
        \includegraphics[width=\linewidth]{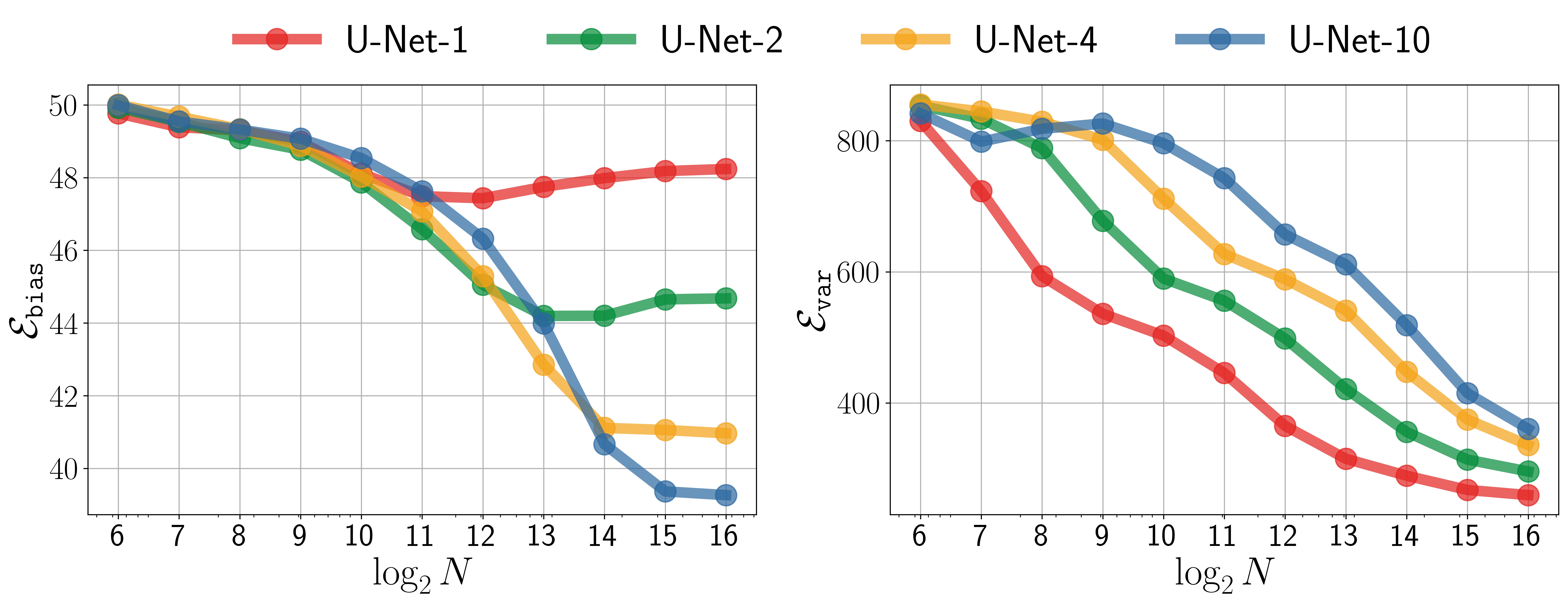}
        \caption{}
    \end{subfigure}
    \caption{\textbf{Bias–Variance Trade-off under DINOv2 descriptor.} (a) plots the generalization error $\PFDg$, bias $\PFDb$, and variance $\PFDv$ across different network architectures with a fixed training sample size of $N = 2^{16}$. (b) shows $\PFDb$ and $\PFDv$ as functions of the number of training samples $N$ for various network architectures.}
    \label{fig:bias-variance-decomposition-dinov2}
\end{center}
\end{figure} 

%% file: arxiv_deepthink/appendix/PFDm_discussion.tex
\section{\texorpdfstring{Further Discussions of $\PFDm$}{Further Discussions of PFDm}} \label{app:estimate_m_score}

In this section, we present the mathematical formulation for estimating $\PFDm$ and compare it with the existing memorization metric.

\noindent \textbf{Empirically estimate $\PFDm$.} As described in \Cref{def:pfd} and \Cref{def:PFD_memorization_generalization}, estimating $\PFDm$ requires access to the mapping $\bm \Phi_{\pemp}(\cdot)$. According to \Cref{eq:n2i_integral}, this mapping is determined by the score function of the empirical distribution, denoted as $\nabla \log \hat{p}_t(\bm x_t)$. Based on prior works \cite{EDM, zhang2023emergence, gu2023memorization}, the score function of the empirical distribution has a closed-form expression:
\begin{align} 
\begin{split}
\nabla \log \hat{p}_t(\bm x_t)
= \frac{1}{T^2}\left(\frac{\E_{\bm x \sim \pemp}[\mathcal{N}(\bm x_t;\bm x, T^2 \bm I_n) \cdot \bm x]}{\E_{\bm x \sim \pemp}[\mathcal{N}(\bm x_t;\bm x, T^2 \bm I_n)]} - \bm x_t\right),
\end{split}
\end{align}   
where $\pemp(\bm x) = \frac{1}{N} \sum_{i=1}^N \delta(\bm x - \bm y^{(i)})$ corresponds to the empirical distribution over the training dataset ${\bm y^{(i)}}_{i = 1}^{N}$. This formulation allows us to numerically compute $\nabla \log \hat{p}_t(\bm x_t)$ for any given $t$. Subsequently, we can use a numerical solver to estimate the integral in \Cref{eq:n2i_integral}, thereby enabling the estimation of $\PFDm$.

\noindent \textbf{Comparison between existing memorization metric and $\PFDm$.}
Previous works \cite{yoon2023diffusion, zhang2023emergence} define memorization metirc as:
\begin{equation}
    \texttt{M Distance}\left(p_{\bm \theta}\right) \coloneqq \mathbb{E}_{\bm x_T} \left[\min\limits_{\bm x \sim \pemp}\left\| \bm \Psi \left(\bm x\right) - \bm \Psi \circ \bm \Phi_{p_{\bm \theta}}  \left(\bm x_T\right) \right\|_2\right],
\end{equation}
A generated sample $\bm \Phi_{p_{\bm \theta}}\left(\bm x_T\right)$ is a memorized sample if it is close enough to one of the sample $\bm x$ from $\pemp$. It is easy to show that $\PFDm$ is a more strict metric than $\texttt{M Distance}$, i.e. "$\PFDm\left(p^{\bm \theta}\right) = 0$" is a sufficient but not necessary condition for "$\texttt{M Distance}\left(p^{\bm \theta}\right) = 0$". We propose $\PFDm$ in order to unify the definition of memorization and generalization.

%% file: arxiv_deepthink/appendix/ablation_study.tex
\section{Ablation Study}\label{app:ablation_study}

In this section, we present ablation studies on the evaluation protocol, examining the effects of different noise schedulers and sampling methods (\Cref{app:diff_sample_algorithm}), image descriptors (\Cref{app:diff_embedding_feature}), sample sizes for evaluation (\Cref{app:diff_sampling_number}), and teacher models (\Cref{app:diff_teacher_model}).

\subsection{Sampling Methods}\label{app:diff_sample_algorithm}

In this subsection, we present ablation studies on various noise schedulers and sampling strategies. Specifically, we evaluate the performance of the following methods: Variance Preserving (VP) \cite{scoreSDE}, Variance Exploding (VE) \cite{scoreSDE}, iDDPM \cite{nichol2021improved} + DDIM \cite{song2021denoising}, and EDM \cite{EDM}. The specific form of $f(t), g(t)$ used in each approach are detailed in Table 1 of \cite{EDM}. Additionally, each method also differs in its choice of ODE solver and timestep discretization strategy. For sampling, we use 256 steps for VP, 1000 for VE, 100 for iDDPM + DDIM, and 18 for EDM. All experiments are conducted under the evaluation protocol described in \Cref{sec:gen-eval}, where we estimate the $\PFDg$ under different training samples $N$. The student models use the UNet-10 architecture. During the ablation study, both the teacher and student models use the same sampling method\footnote{Note that the noise scheduler used for sampling could differ from that used during training.} as specified above.

As shown in \Cref{fig:different_scheduler}, different samplers yield highly consistent results, demonstrating that $\PFD$ can be extended to various noise schedules, i.e., different choices of $f(t)$ and $g(t)$.

\begin{figure}[t]
\begin{center}
    \includegraphics[width = \linewidth]{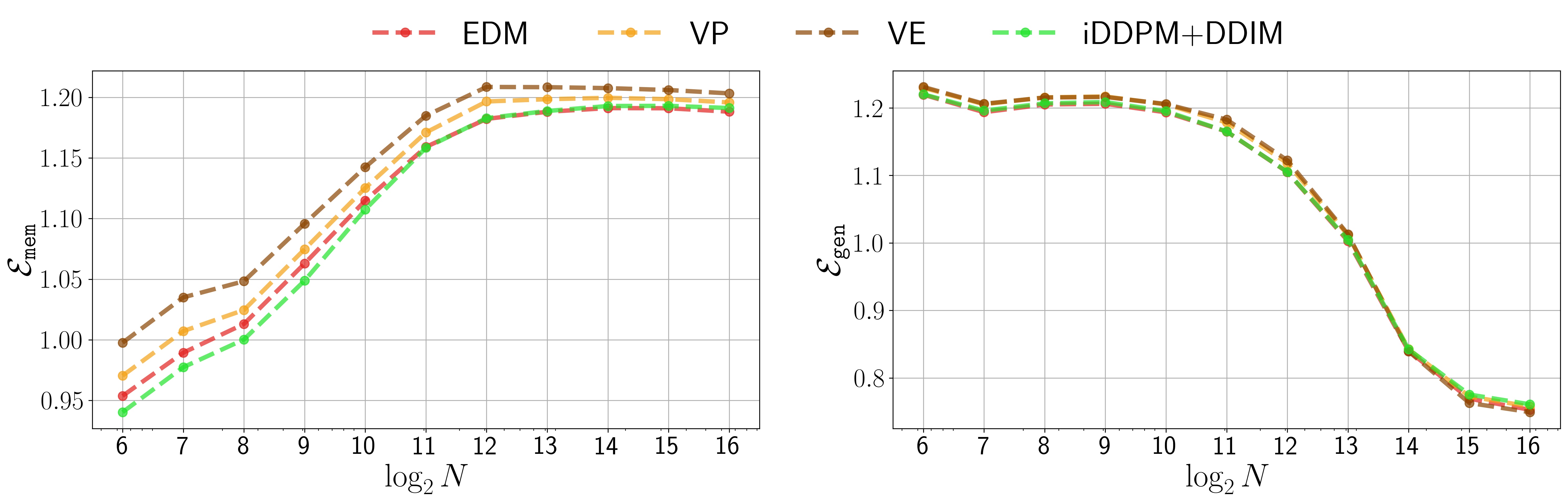}
    \caption{\textbf{Comparison of different sampling methods.} $\PFDm$ and $\PFDg$ plotted against $\log_2(N)$ for different sampling methods, including: EDM, VP, VE, iDDPM+DDIM.}
    \label{fig:different_scheduler}
\end{center}
\end{figure} 

\subsection{Image Descriptors} \label{app:diff_embedding_feature}

In this subsection, we present ablation studies on the image descriptor $\bm{\Psi}$ used in \Cref{eq:pfd}. The descriptors evaluated include DINOv2 \cite{dinov2}, InceptionV3 \cite{inceptionv3}, CLIP \cite{clip}, SSCD \cite{sscd}, and the identity function. All experiments follow the evaluation protocol described in \Cref{sec:gen-eval}, where we estimate both $\PFDm$ and $\PFDg$ across varying training sample sizes $N$ and different student model architectures: U-Net-1, U-Net-2, U-Net-4, and U-Net-10.

As shown in \Cref{fig:different_descriptor}, different feature embeddings reveal a consistent trend in the memorization-to-generalization (MtoG) transition across various U-Net architectures. With limited training samples, smaller models exhibit lower generalization scores. Conversely, with sufficient training data, larger models tend to have lower generalization scores. When comparing with $\PFDg$ measured in pixel space (i.e., using the identity function as the descriptor), we observe that $\PFDg$ values are nearly identical across diffusion architectures when $N \geq 2^{15}$. In this regime, all models have learned low-level image features such as color and structure; however, only the larger models capture high-level perceptual details. Because pixel-space measurements fail to reflect these high-level features, they yield similar $\PFDg$ values regardless of model size. Therefore, it is better to evaluate $\PFDg$ in a feature space, which better captures perceptual differences between models.

Different feature descriptors mainly differ in the generalization regime. Specifically, $\PFDg$ varies the most across architectures when using the DINOv2 descriptor, and the least when using the SSCD descriptor. This is because each descriptor capture different aspects of the image. SSCD focuses on detecting duplicate content and is more sensitive to low-frequency features, while DINOv2 emphasizes perceptual quality and captures high-frequency features. Diffusion models with limited capacity tend to learn low-frequency information first, as it it more easier to learn \cite{wang2025analytical}. As a result, under the SSCD descriptor, different architectures show more similar $\PFDg$ values, since they are all primarily capturing the same low-frequency information in the early training stages.

\begin{figure}[t]
\begin{center}
    \includegraphics[width = \linewidth]{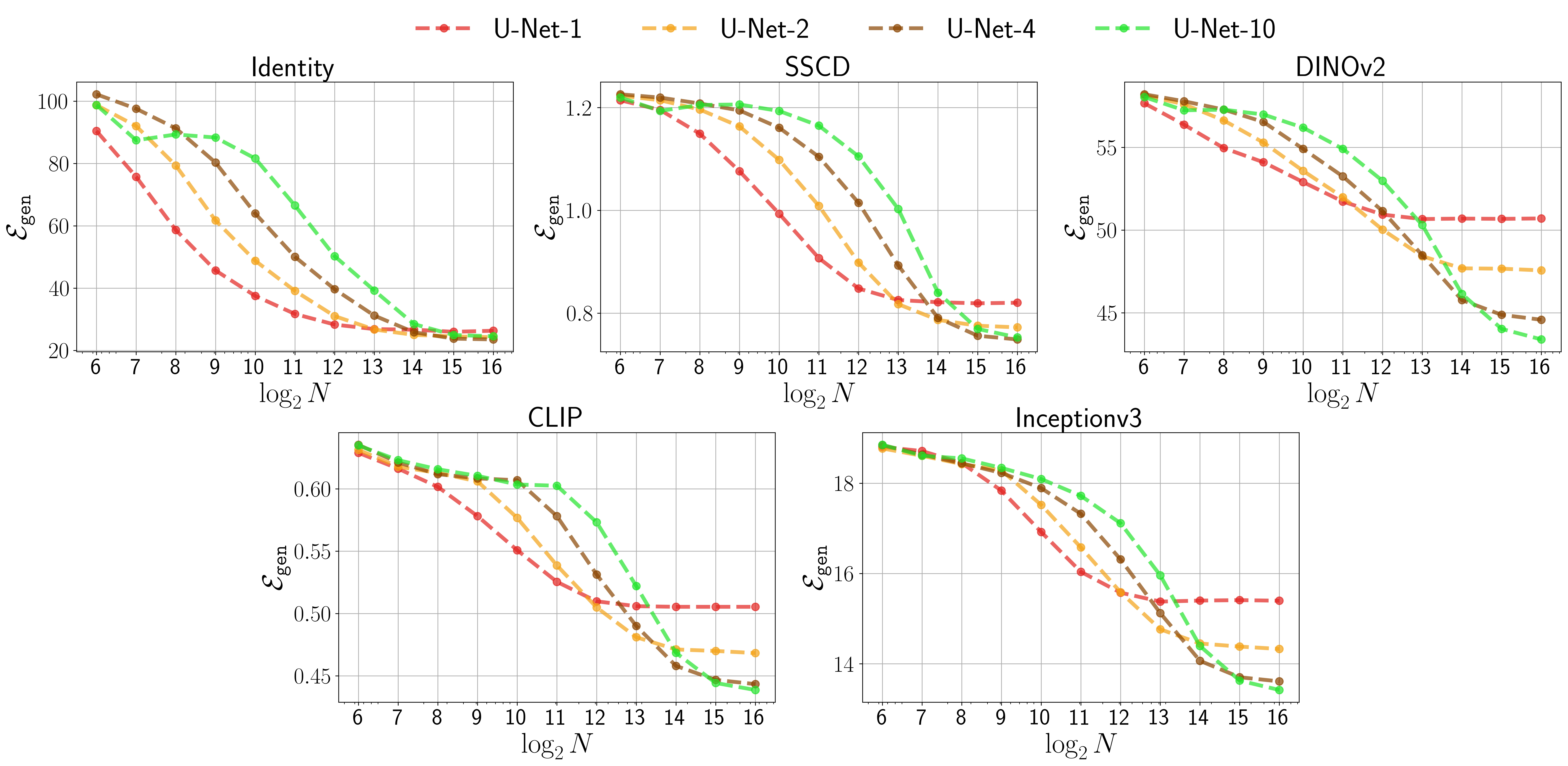}
    \caption{\textbf{Comparison between different image descriptors.} $\PFDg$ plotted against $\log_2(N)$ for a range of U-Net architectures (U-Net-1, U-Net-2, U-Net-4, U-Net-10) using different image descriptors, including identity function, SSCD, DINOv2, CLIP, Inceptionv3.}
    \label{fig:different_descriptor}
\end{center}
\end{figure} 

\begin{figure}[htbp]
\begin{center}
    \includegraphics[width = .6\linewidth]{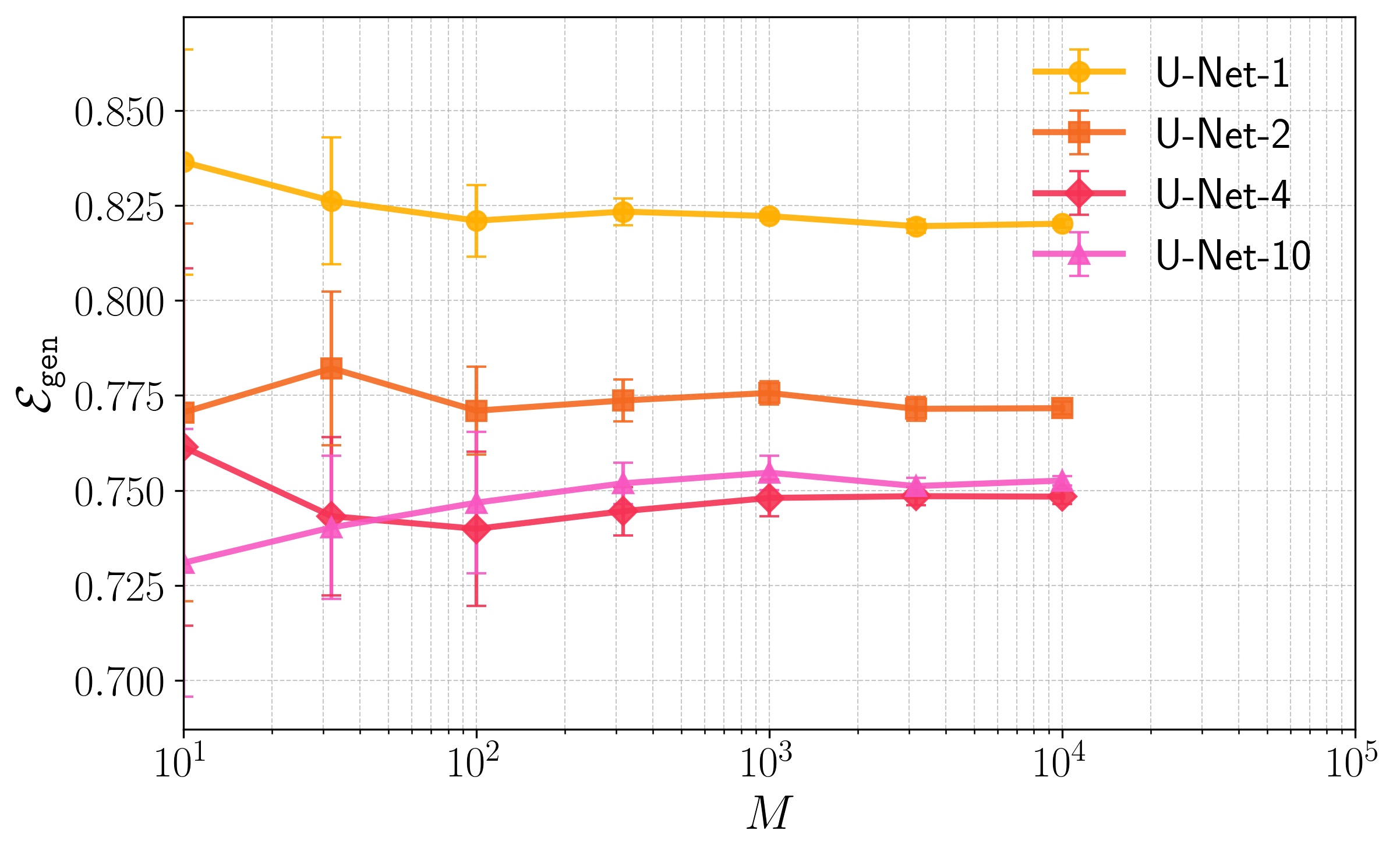}
    \caption{\textbf{Comparison across evaluation sample sizes.} The mean and variance of $\PFDg$ are plotted against the number of evaluation samples $M$ for various U-Net architectures (U-Net-1, U-Net-2, U-Net-4, U-Net-10), with a fixed number of training samples $N = 2^{16}$.}
    \label{fig:different_eval_sample}
\end{center}
\end{figure} 

\subsection{Evaluation Sample Number} \label{app:diff_sampling_number}

In this subsection, we present ablation studies on the number of samples $M$ used by $\hat{\PFD}$ to approximate $\PFD$, as defined in \Cref{eqn:empirical-pfd}. All experiments follow the evaluation protocol described in \Cref{sec:gen-eval}, where we estimate $\PFDg$ across varying training sample sizes $N$ and different student model architectures: U-Net-1, U-Net-2, U-Net-4, and U-Net-10. We vary $M \in \{10, 32, 100, 316, 1000, 3163, 10000\}$, and for each setting, generate 5 independent sets of $\{\bm{x}^{(i)}_{\texttt{gen}, T}\}_{i=1}^{M}$ initial noise estimate $\PFDg$, computing both the mean and variance. 

As shown in \Cref{fig:different_eval_sample}, the variance of $\PFDg$ approaches zero as $M$ increases to 10,000, indicating that when $M \geq 10000$, the empirical estimate of $\PFDg$ converges to its value over the underlying distribution. This result holds consistently across different model architectures.

\subsection{Teacher Model Architecture}\label{app:diff_teacher_model}

\begin{figure}[htbp]
    \centering
    \includegraphics[width=.4\linewidth]{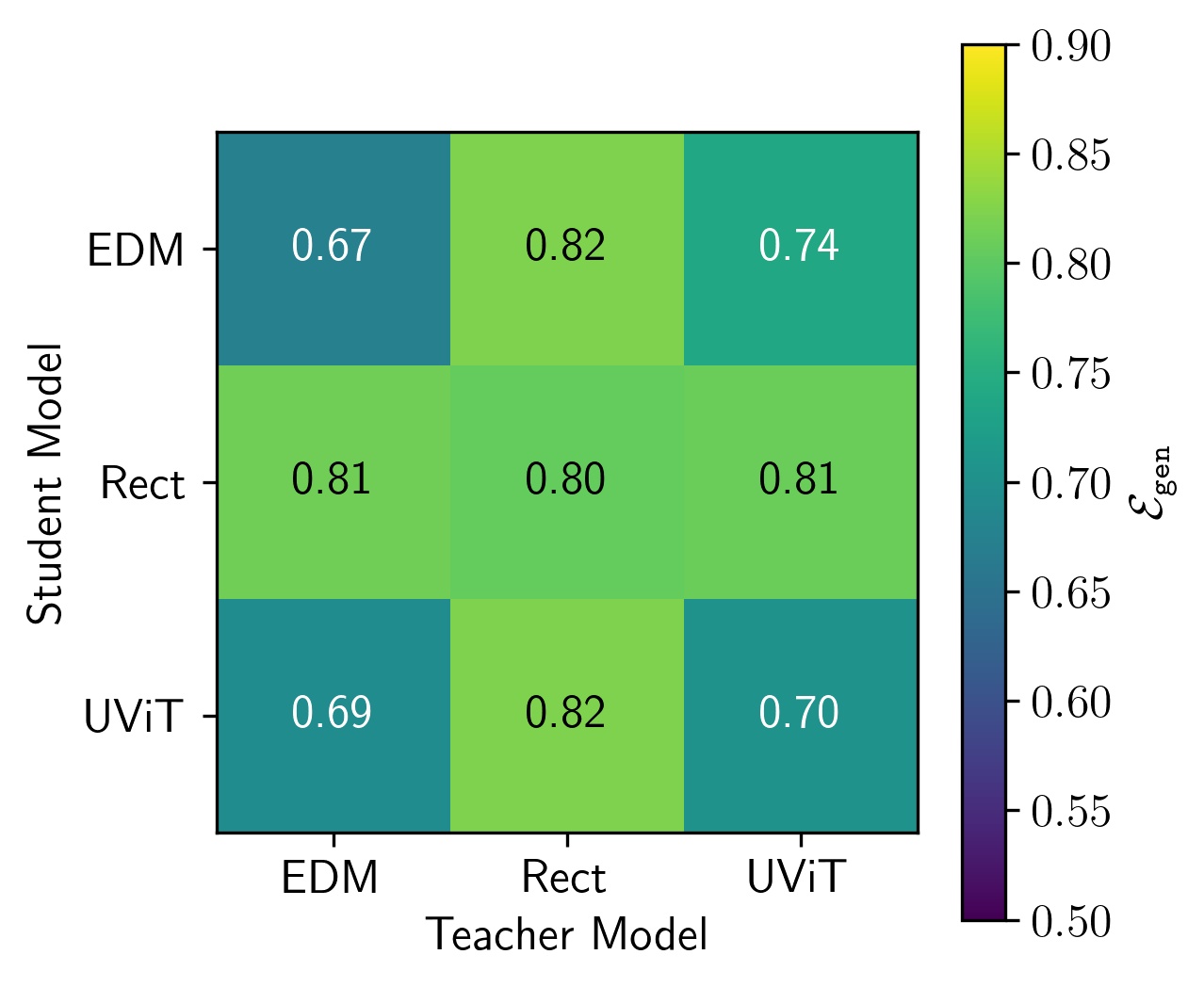}
    \caption{\textbf{Comparison of different teacher models.}The figure shows the $\PFDg$ values for various student models (EDM, Rect, UViT) trained using different teacher models (EDM, Rect, UViT), with a fixed training data size of $N = 2^{16}$.}
     \label{fig:different_teacher_model}
\end{figure}

We end this section by examining how different teacher models affect the evaluation protocol. Specifically, we consider three types of diffusion models: EDM, Rectified Flow (Rect) \cite{RF}, and UViT. Using the CIFAR-10 dataset, we train three teacher models, one for each of these diffusion types. For each teacher model, we then evaluate all three diffusion models as student models. We report their corresponding $\PFDg$ values. Both teacher and student models use the same sampling method, the second-order Heun solver with 18 steps. 

As shown in \Cref{fig:different_teacher_model}, the $\PFDg$ is approximately 0.7 when both the student and teacher models are selected from EDM or UViT. However, $\PFDg$ increases to around 0.8 when either the student or teacher model is Rect. According to its original paper, Rect has the poorest generation quality among the three, as measured by $\mathtt{FID}$. This suggests that the teacher model should possess strong generative performance to serve as an underlying distribution that is close to the real-world data distribution. Therefore, in this paper, we adopt EDM as the teacher model, as it achieves the lowest $\mathtt{FID}$ among the three models.